\documentclass[11pt,letterpaper]{article}
\usepackage[margin=1in]{geometry}
\usepackage{subfigure}
\usepackage{comment}
\usepackage{booktabs}
\usepackage{graphicx}
\usepackage{diagbox}
\usepackage{float}
\usepackage{amsmath,amsfonts,graphicx,amssymb,bm,amsthm,amssymb}
\usepackage{algorithm,algorithmicx}
\usepackage[noend]{algpseudocode}
\usepackage{fancyhdr}
\usepackage[pdf]{graphviz}
\usepackage{dot2texi}
\usepackage{tikz}
\usetikzlibrary{shapes,arrows}
\usetikzlibrary{arrows,automata}
\usepackage{hyperref}
\usepackage{cleveref}
\usepackage{natbib}
\usepackage{enumerate}
\usepackage{centernot}
\algblockdefx[Foreach]{Foreach}{EndForeach}[1]{\textbf{foreach} #1 \textbf{do}}{\textbf{end foreach}}

\usepackage{graphicx,pstricks,listings,stackengine,tikz-layers}
\usetikzlibrary{positioning,chains,fit,shapes,calc}
\definecolor{myblue}{RGB}{80,80,160}
\definecolor{mygreen}{RGB}{80,160,80}

\newcounter{counter_exm}
\newtheorem{theorem}{Theorem}
\newtheorem{lemma}[theorem]{Lemma}

\newtheorem{corollary}[theorem]{Corollary}
\newtheorem{definition}[theorem]{Definition}
\newtheorem{example}{Example}[section]
\newcommand\E{\mathbb{E}}
\newcommand\loweps{\underline{\varepsilon}}
\newcommand\higheps{\overline{\varepsilon}}
\newcommand\lowset{\underline{B}}
\newcommand\highset{\overline{B}}
\newcommand\lowelement{\underline{i}}
\newcommand\highelement{\overline{i}}
\newcommand{\email}[1]{\href{mailto:#1}{\nolinkurl{#1}}}
\usepackage{indentfirst}

\title{Fair Grading Algorithms for Randomized Exams}
\author{
Jiale Chen\thanks{Department of Management Science and Engineering, Stanford University. Email: \email{jialec@stanford.edu}.}
\and
Jason Hartline\thanks{Department of Computer Science, Northwestern University. Email: \email{hartline@northwestern.edu}.}
\and
Onno Zoeter\thanks{Booking.com. Email: \email{Onno.zoeter@booking.com}.}
}
\date{}

\begin{document}

\maketitle
\begin{abstract}
This paper studies grading algorithms for randomized exams. In a randomized exam, each student is asked a small number of random questions from a large question bank. The predominant grading rule is simple averaging, i.e., calculating grades by averaging scores on the questions each student is asked, which is fair ex-ante, over the randomized questions, but not fair ex-post, on the realized questions.  The fair grading problem is to estimate the average grade of each student on the full question bank.  The maximum-likelihood estimator for the Bradley-Terry-Luce model on the bipartite student-question graph is shown to be consistent with high probability when the number of questions asked to each student is at least the cubed-logarithm of the number of students. In an empirical study on exam data and in simulations, our algorithm based on the maximum-likelihood estimator significantly outperforms simple averaging in prediction accuracy and ex-post fairness even with a small class and exam size.
\end{abstract}

\section{Introduction}

A common approach for deterring cheating in online examinations is to
assign students random questions from a large question bank.  This
random assignment of questions with heterogeneous difficulties leads to
different overall difficulties of the exam that each student faces.
Unfortunately, the predominant grading rule -- simple averaging --
averages all question scores equally and results in an unfair grading
of the students. This paper develops a grading algorithm that utilizes
structural information of the exam results to infer student abilities
and question difficulties.  From these abilities and difficulties,
fairer and more accurate grades can be estimated.  This grading
algorithm can also be used in the design of short exams that
maintain a desired level of accuracy.

During the COVID-19 pandemic, learning management systems (LMS) like
Blackboard, Moodle, Canvas by Instructure, and D2L have benefited
worldwide students and teachers in remote learning
\citep{razaSocialIsolationAcceptance2021}. The current exam module in
these systems includes four steps. In the first step, the instructor
provides a large question bank.  In the second step, the system
assigns each student an independent random subset of the
questions. (Assigning each student an independent random subset of the
questions helps mitigate cheating.)  In the third step, students
answer the questions. In the last step, the system grades each student
proportionally to her accuracy on assigned questions, i.e., by {\em
  simple averaging}.

While randomizing questions and grading with simple averaging is ex-ante fair,
it is not generally ex-post fair.  When questions in the
question bank have varying difficulties, then by random chance a student
could be assigned more easy questions than average or more hard
questions than average.  Ex-post in the random assignment of questions
to students, the simple averaging of scores on each question allows
variation in question difficulties to manifest as ex-post
unfairness in the final grades.

The aim of this paper is to understand grading algorithms that are fair
and accurate. Given a bank of possible questions, a benchmark for
both fairness and accuracy is the counterfactual grade that a student
would get if the student was asked all of the questions in the
question bank. Exams that ask fewer questions to the students may be
inaccurate with respect to this benchmark and the inaccuracy may vary
across students and this variation is unfair. This benchmark allows
for both the comparison of grading algorithms and the design of randomized
exams, i.e., the method for deciding which questions are asked to
which students.

The grading algorithms developed in this paper are based on the
Bradley-Terry-Luce model \citep{Bradley1952-nz} on bipartite
student-question graphs.  This model is also studied in the psychology
literature where it is known as the Rasch model
\citep{raschProbabilisticModelsIntelligence1993}.  This model views
the student answering process as a noisy comparison between a
parameter of the student and a parameter of the
question. Specifically, there is a merit value vector $u$ which describes
the student abilities and question difficulties and is unknown to
the instructor. The probability that student $i$ answers
question $j$ correctly is defined to be
\[f(u_i-u_j)=\frac{\exp(u_i)}{\exp(u_i)+\exp(u_j)},\]
where $f(x)=\frac{1}{1+\exp(-x)}$, and $u_i$, $u_j$ represents the merit value of student $i$ and question $j$ respectively.

The paper develops a grading algorithm that is based on the maximum
likelihood estimator $\boldsymbol{u^*}$ of the merit vector. Compared to simple
averaging which only focuses on student in-degrees and out-degrees, our
grading algorithm incorporates more structural information about the exam
result and, as we show, reduces ex-post unfairness.

\paragraph{Results.}

Our theoretical analysis considers a sequence of distributions over
random question assignment graphs indexed by $n$ and $m$
by setting the number of students to $n$
and number of questions in the question bank to $m\geq n$ and
assigning $d_{n,m}$ random questions uniformly and independently to each
student. The exam
result can be represented by a directed graph, where an edge from a
student to a question represents a correct answer and the opposite
direction represents an incorrect answer.  We prove that the maximum
likelihood estimator exists and is unique within a strongly connected
component (\Cref{thm:equiv}). 
Let $\alpha_{n,m}=\max_{1\leq i,j\leq n+m} u_i-u_j$ be the largest
difference between any pair of merits.  We prove that if
 \[\frac{\exp(\alpha_{n,m})(n+m)\log(n+m)}
    {nd_{n,m}}\rightarrow 0\quad (n,m\rightarrow\infty),\]
then the probability that the exam result graph is strongly connected goes to 1 (\Cref{thm:connectivity}). Thus, the existence and uniqueness of the MLE are
guaranteed under the model. We also prove that if
\begin{equation}
    \exp\left(2(\alpha_{n,m}+1)\right)
    \Delta_{n,m}\rightarrow 0\quad(n,m\rightarrow\infty),
\end{equation}
where $\Delta_{n,m}= \sqrt{\frac{m\log^3 (n+m)}
    {nd_{n,m}\log^2 \left(\frac{n}{m}d_{n,m}\right)}}$,
then the MLEs are uniformly consistent, i.e.,
$\|\boldsymbol{u}^*-\boldsymbol{u}\|_\infty\stackrel{\mathbb{P}}{\longrightarrow}
0$ (\Cref{thm:consistency}).  These theoretical results complement the empirical and simulation
results from the literature on the Rasch model with random missing data.  Our analysis is similar
to that of \citet{hanAsymptoticTheorySparse2020} which studies
Erd\"os-R\'enyi random graphs.

Our empirical analysis considers a study of grading algorithms on both
anonymous exam data and numerical simulations.  The exam data set
consists of 22 questions and 35 students with all students answering
all questions.  From this data set, randomized exams with fewer than 22
questions can be empirically studied and grading algorithms can be
compared.  Our algorithm outperforms simple averaging when students
are asked at least seven questions.  We fit the model parameters
to this real-world dataset and run numerical simulations with the
resulting generative model.  With these simulations, we compare our
algorithm and simple averaging on ex-post bias and ex-post error,
two notion of ex-post unfairness.
For example, when each of the 35 students answers
a random 10 of the 22 questions, we find that the
expected maximum ex-post bias of simple averaging is about $100$ times higher than
that of our algorithm.  The expected output of simple averaging
has about 13\% expected deviation from the benchmark for the most unlucky student,
which would probably lead to a different letter grade for the students,
while the deviation is only about 1.6\% for our algorithm. In the same setting, we found that
our algorithm achieves a factor of 8 percent smaller ex-post error, which is a noisier
concept of ex-post unfairness. After the decomposition of ex-post error into
ex-post bias and variance, we found that our algorithm achieves
a significantly smaller ex-post bias with the cost of a slightly larger variance of the
output, and in combination it reduces the ex-post error.

We also evaluate a related simple problem of exam design via simulation in \Cref{sec:exam design problem}.
Given an infinite question bank, we sample a fixed number of active questions,
then we create a randomized exam with five questions for each student drawn
from this fixed number of active questions.  When there
are only five active questions, our grading algorithm and simple averaging
coincide as all students are asked all active questions.  When there
are an infinite number of active questions the algorithms again
coincide as no two students are asked the same question and our
algorithm and simple averaging are the same.  By simulation we
consider the ex-post bias as a function of the number of active questions and
find that the optimal number of questions is about six to nine for maximum
ex-post bias and about ten to fifteen for average ex-post bias.

\paragraph{Illustrative Example.} 

\Cref{fig:example} illustrates the unfairness that can arise in simple
averaging and the intuition for how our algorithm improves fairness.
In the fair exam grading problem, the instructor first assigns
questions to students according to an undirected student-question
bipartite graph, a.k.a., the task assignment graph
(\Cref{fig:example:assignment}).  The exam result can be represented
by a directed student-question bipartite graph, a.k.a., the exam
result graph (\Cref{fig:example:result}), where a directed edge from a
student to a question represents a correct answer and an opposite
direction represents an incorrect answer. Given the exam result graph,
the instructor uses a grading algorithm to estimate the average
accuracy of students on the whole question bank.

We take student S2 as an example to see how simple averaging, as one
specific grading rule, grade students
(\Cref{fig:example:averaging}). Simple averaging observes that student
S2 has one correct answer and one incorrect answer, in other words,
the corresponding vertex has in-degree one and out-degree one. Thus
simple averaging predicts that the probability of student S2 answering
the remaining question Q3 correctly is 0.5. The reason we believe 0.5
is not a good prediction is, in this exam, every student assigned
question Q3 answers it correctly, including student S5 and S6, who
give incorrect answers to question Q2.  We can infer that Q3 is a
relatively easy question. On the other hand, student S2 who can answer
question Q2 correctly has a relatively higher ability among the
class. Therefore, it is reasonable to believe that student S2 can give
a correct answer to question Q3 and, thus, simple averaging
underestimates her grade.

We show how our algorithm analyze the missing edge between student S2
and question Q3 in this example (\Cref{fig:example:my}).  Our
algorithm finds that student S2 belongs to the strongly connected
component \{S1,S2,Q1,Q2\}, while question Q3 belongs to \{Q3\} and the
missing edge goes across two comparable components. As a property of
the graph, any directed path between student S2 and question Q3 goes
from the student to the question. Our grading rule takes it as a
strong evidence of the S2's ability higher than Q3's difficulty
and predicts that student S2 can answer question Q3 correctly with
probability one.

\begin{figure}[!h]
    \subfigure[Task Assignment Graph]{
        \begin{minipage}{0.45\linewidth}
            \centering
            \scalebox{0.8}{\begin{tikzpicture}[
thick,
roundnode/.style={circle, draw=green!60, fill=green!5, very thick, minimum size=7mm},
squarednode/.style={rectangle, draw=red!60, fill=red!5, very thick, minimum size=7.5mm},
smallroundnode/.style={circle, draw=green!60, fill=green!5, very thick, minimum size=3mm},
smallsquarednode/.style={rectangle, draw=red!60, fill=red!5, very thick, minimum size=3.4mm},
]
%Nodes
\node[roundnode]    (Q1)    at  (-2, 3)     {Q1};
\node[roundnode]    (Q2)    at  (2, 3)      {Q2};
\node[roundnode]    (Q3)    at  (0, -1.5)     {Q3};
\node[squarednode]  (S1)    at  (0, 4)      {S1};
\node[squarednode]  (S2)    at  (0, 2)      {S2};
\node[squarednode]  (S3)    at  (-3, 0.5)     {S3};
\node[squarednode]  (S4)    at  (-1, 0.5)     {S4};
\node[squarednode]  (S5)    at  (1, 0.5)      {S5};
\node[squarednode]  (S6)    at  (3, 0.5)      {S6};

%Lines
\draw[-] (S1) -- (Q1);
\draw[-] (S1) -- (Q2);
\draw[-] (S2) -- (Q1);
\draw[-] (S2) -- (Q2);
\draw[-] (S3) -- (Q1);
\draw[-] (S3) -- (Q3);
\draw[-] (S4) -- (Q1);
\draw[-] (S4) -- (Q3);
\draw[-] (S5) -- (Q2);
\draw[-] (S5) -- (Q3);
\draw[-] (S6) -- (Q2);
\draw[-] (S6) -- (Q3);

\matrix [draw,right] at (3.2, 4) {
  \node [smallroundnode,anchor=east,label=right:Question] {}; \\
  \node [smallsquarednode,anchor=east,label=right:Student] {}; \\
};
\end{tikzpicture}}
            \label{fig:example:assignment}
        \end{minipage}
    }\hfill
    \subfigure[Exam Result Graph]{
        \begin{minipage}{0.45\linewidth}
            \centering
            \scalebox{0.8}{\begin{tikzpicture}[
thick,
roundnode/.style={circle, draw=green!60, fill=green!5, very thick, minimum size=7mm},
squarednode/.style={rectangle, draw=red!60, fill=red!5, very thick, minimum size=7.5mm},
smallroundnode/.style={circle, draw=green!60, fill=green!5, very thick, minimum size=3mm},
smallsquarednode/.style={rectangle, draw=red!60, fill=red!5, very thick, minimum size=3.4mm},
->,shorten >= 1pt,shorten <= 1pt,
]
%Nodes
\node[roundnode]    (Q1)    at  (-2, 3)     {Q1};
\node[roundnode]    (Q2)    at  (2, 3)      {Q2};
\node[roundnode]    (Q3)    at  (0, -1.5)     {Q3};
\node[squarednode]  (S1)    at  (0, 4)      {S1};
\node[squarednode]  (S2)    at  (0, 2)      {S2};
\node[squarednode]  (S3)    at  (-3, 0.5)     {S3};
\node[squarednode]  (S4)    at  (-1, 0.5)     {S4};
\node[squarednode]  (S5)    at  (1, 0.5)      {S5};
\node[squarednode]  (S6)    at  (3, 0.5)      {S6};

%Lines
\draw[->] (S1) -- (Q1);
\draw[<-] (S1) -- (Q2);
\draw[<-] (S2) -- (Q1);
\draw[->] (S2) -- (Q2);
\draw[<-] (S3) -- (Q1);
\draw[->] (S3) -- (Q3);
\draw[<-] (S4) -- (Q1);
\draw[->] (S4) -- (Q3);
\draw[<-] (S5) -- (Q2);
\draw[->] (S5) -- (Q3);
\draw[<-] (S6) -- (Q2);
\draw[->] (S6) -- (Q3);

\matrix [draw,right] at (3.2, 4) {
  \node [smallroundnode,anchor=east,label=right:Question] {}; \\
  \node [smallsquarednode,anchor=east,label=right:Student] {}; \\
};
\end{tikzpicture}}
            \label{fig:example:result}
        \end{minipage}
    }
    \subfigure[simple averaging]{
        \begin{minipage}{0.45\linewidth}
            \centering
            \scalebox{0.8}{\begin{tikzpicture}[
thick,
roundnode/.style={circle, draw=green!60, fill=green!5, very thick, minimum size=7mm},
squarednode/.style={rectangle, draw=red!60, fill=red!5, very thick, minimum size=7.5mm},
smallroundnode/.style={circle, draw=green!60, fill=green!5, very thick, minimum size=3mm},
smallsquarednode/.style={rectangle, draw=red!60, fill=red!5, very thick, minimum size=3.4mm},
->,shorten >= 1pt,shorten <= 1pt,
every fit/.style={ellipse,draw,inner sep=5pt,text width=1.5cm}
]
%Nodes
\node[roundnode]    (Q1)    at  (-2, 3)     {Q1};
\node[roundnode]    (Q2)    at  (2, 3)      {Q2};
\node[roundnode]    (Q3)    at  (0, -1.5)     {Q3};
\node[squarednode]  (S1)    at  (0, 4)      {S1};
\node[squarednode]  (S2)    at  (0, 2)      {S2};
\node[squarednode]  (S3)    at  (-3, 0.5)     {S3};
\node[squarednode]  (S4)    at  (-1, 0.5)     {S4};
\node[squarednode]  (S5)    at  (1, 0.5)      {S5};
\node[squarednode]  (S6)    at  (3, 0.5)      {S6};

%Lines
\draw[->] (S1) -- (Q1);
\draw[<-] (S1) -- (Q2);
\draw[<-] (S2) -- (Q1);
\draw[->] (S2) -- (Q2);
\draw[<-] (S3) -- (Q1);
\draw[->] (S3) -- (Q3);
\draw[<-] (S4) -- (Q1);
\draw[->] (S4) -- (Q3);
\draw[<-] (S5) -- (Q2);
\draw[->] (S5) -- (Q3);
\draw[<-] (S6) -- (Q2);
\draw[->] (S6) -- (Q3);

\matrix [draw,right] at (3.2, 4) {
  \node [smallroundnode,anchor=east,label=right:Question] {}; \\
  \node [smallsquarednode,anchor=east,label=right:Student] {}; \\
};

\node [myblue,fit=(S2)] {};
\draw[dashed,-] (S2) -- node[right] {.5} ++ (Q3);
\end{tikzpicture}}
            \label{fig:example:averaging}
        \end{minipage}
    }\hfill
    \subfigure[Our method]{
        \begin{minipage}{0.45\linewidth}
            \centering
            \scalebox{0.8}{\begin{tikzpicture}[
thick,
roundnode/.style={circle, draw=green!60, fill=green!5, very thick, minimum size=7mm},
squarednode/.style={rectangle, draw=red!60, fill=red!5, very thick, minimum size=7.5mm},
smallroundnode/.style={circle, draw=green!60, fill=green!5, very thick, minimum size=3mm},
smallsquarednode/.style={rectangle, draw=red!60, fill=red!5, very thick, minimum size=3.4mm},
->,shorten >= 1pt,shorten <= 1pt,
every fit/.style={ellipse,draw,inner sep=-3mm, fill=blue!5}
]
%Nodes
\node[roundnode]    (Q1)    at  (-2, 3)     {Q1};
\node[roundnode]    (Q2)    at  (2, 3)      {Q2};
\node[roundnode]    (Q3)    at  (0, -1.5)     {Q3};
\node[squarednode]  (S1)    at  (0, 4)      {S1};
\node[squarednode]  (S2)    at  (0, 2)      {S2};
\node[squarednode]  (S3)    at  (-3, 0.5)     {S3};
\node[squarednode]  (S4)    at  (-1, 0.5)     {S4};
\node[squarednode]  (S5)    at  (1, 0.5)      {S5};
\node[squarednode]  (S6)    at  (3, 0.5)      {S6};

%Lines
\draw[->] (S1) -- (Q1);
\draw[<-] (S1) -- (Q2);
\draw[<-] (S2) -- (Q1);
\draw[->] (S2) -- (Q2);
\draw[<-] (S3) -- (Q1);
\draw[->] (S3) -- (Q3);
\draw[<-] (S4) -- (Q1);
\draw[->] (S4) -- (Q3);
\draw[<-] (S5) -- (Q2);
\draw[->] (S5) -- (Q3);
\draw[<-] (S6) -- (Q2);
\draw[->] (S6) -- (Q3);

\matrix [draw,right] at (3.2, 4) {
  \node [smallroundnode,anchor=east,label=right:Question] {}; \\
  \node [smallsquarednode,anchor=east,label=right:Student] {}; \\
};
\begin{scope}[on background layer]
\node [myblue,fit=(S1)(S2)(Q1)(Q2)] {};
\end{scope}
\draw[dashed,-] (S2) -- node[right] {1} ++ (Q3);
\end{tikzpicture}}
            \label{fig:example:my}
        \end{minipage}
    }
    \caption{A running example of the exam grading problem}
    \label{fig:example}
\end{figure}
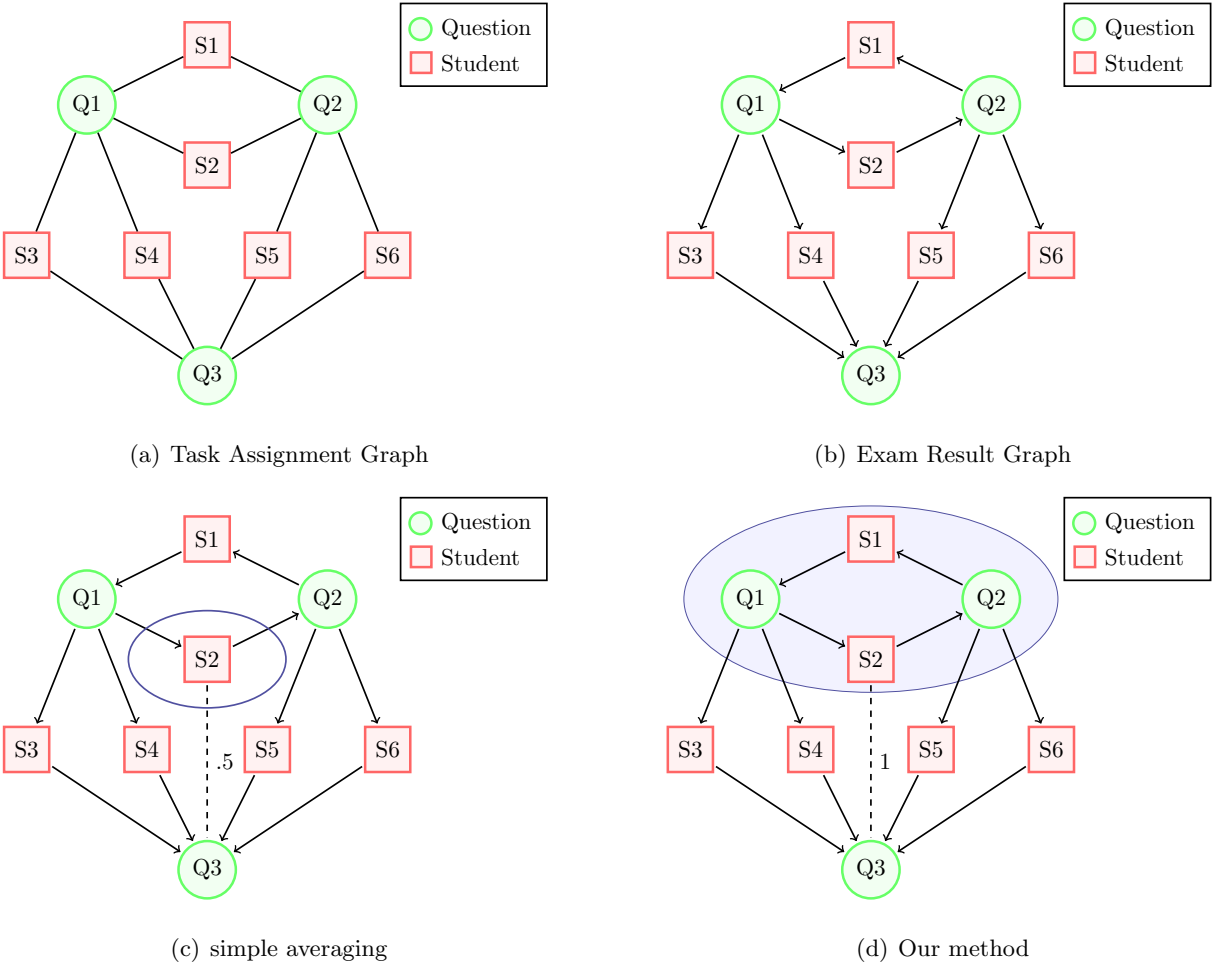
\paragraph{Related Work.}

The literature on peer grading also compares estimation from
structural models and simple averaging.  When peers are assigned to
grade submissions, the quality of peer reviews can vary.  Structural
models can be used to estimate peer quality and calculate grades on
the submissions that put higher weight on peers who give higher
quality reviews.  Alternatively, submission grades can be calculated
by simply averaging the reviews of each peer.  The literature has
mixed results.  \citet{dealfaroCrowdGraderCrowdsourcingEvaluation}
propose an algorithm based on reputation that largely outperforms
simple averaging on synthetic data, and is better on real-world data
when student grading error is not random. \citet{reilyTwoPeersAre2009}
and \citet{hamerMethodAutomaticGrade2005a} also point out that
sophisticated aggregation improves the accuracy compared to simple
averaging and also helps to avoid rogue strategies including laziness
and aggressive grading.  On the other hand,
\citet{sajjadiPeerGradingCourse2016} show that statistical and machine
learning methods do not perform better than simple averaging on their
dataset.  In contrast our result that structural models outperform
simple averaging is replicated on several data sets.  We believe this
difference with the peer grading literature is due to differences in
the degrees of the bipartite graphs considered.  The exam grading
graphs are higher degree than the peer-grading graphs.

In psychometrics, item response theory (IRT) considers mathematical
models that build relationship between unobserved characteristics of
respondents and items and observed outcomes of the responses. The
Rasch model is a commonly used model of IRT that can be applied to
psychometrics, educational research
\citep{raschProbabilisticModelsIntelligence1993}, health sciences
\citep{bezruczkoRaschMeasurementHealth2005}, agriculture
\citep{moralCharacterizationSoilFertility2017}, and market research
\citep{bechtelGeneralizingRaschModel1985}. Previous simulation studies
showed that among different item parameter estimation methods for the
Rasch model, the joint maximum likelihood (JML) method and its variants
provides one of the most efficient estimates
\citep{robitzschComprehensiveSimulationStudy2021}, especially with missing data
\citep{waterburyMissingDataRasch2019,endersAppliedMissingData2010}. In
our setting, randomly assignment of questions to students can be seen
as a special case of missing data. With complete data, the condition for the consistency of the maximum likelihood estimators is analyzed  \citep{habermanMaximumLikelihoodEstimates1977,habermanJOINTCONDITIONALMAXIMUM2004}. With missing data, though plenty of works on simulation
exists, there is a lack of theoretical work that proves mathematically
the consistency of the maximum likelihood estimators.

The Rasch model can be regarded as a special case of the
Bradley-Terry-Luce (BTL) model \citep{Bradley1952-nz} for the pairwise
comparison of respondents with items by restricting the comparison
graph to a bipartite graph. For the BTL model with Erd\"os-R\'enyi graph $G(n,
p_n)$, the maximum likelihood estimator (MLE) can be solved by
an efficient algorithm
\citep{zermeloBerechnungTurnierErgebnisseAls1929c,fordSolutionRankingProblem1957,hunterMMAlgorithmsGeneralized2004},
and is proved to be a consistent method in $l_{\infty}$ norm when
$\lim\inf_{n\to\infty} p_n>0$ \citep{simonsAsymptoticsWhenNumber1999,
  yanSPARSEPAIREDCOMPARISONS2012}, and recently when $p_n\geq \frac{\log
  n^3}{n}$ \citep{hanAsymptoticTheorySparse2020} which is close to the
theoretical lower bound of $\frac{\log n}{n}$, below which the
comparison graph would be disconnected with positive probability and
there is no unique MLE.

In this paper, we follow the method of
\citet{hanAsymptoticTheorySparse2020} to prove the consistency of the
Rasch model with missing data, or BTL model with a sparse bipartite
graph, when each vertex in the left part is assigned small number of
random edges to the vertices in the right part. We also propose an
extension of the algorithm that reasonably deals with the cases where
the MLE does not exists.

\citet{fowlerAreWeFair2022} recently studied unfairness detection of the simple averaging under the same randomized exam setting and argue that ``the exams are reasonably fair''. They use certain IRT model to fit exams based on their real-world data, and find that the simple averaging gives grades that are strongly correlated with the students' inferred abilities. They also simulate under the IRT model, over random assignment and the student answering process. The simulation shows that, if given any fixed assignment we consider the absolute error of the students' expected performance over their answering process, the average absolute error over different assignments reaches a 5-percentage bias. We find similar results in our simulation, and design a method to reduce the corresponding error by a factor of ten. Our method solves one of their future directions by adjusting grades of the students based on their exam variant.

All large-scale standardized tests including the Scholastic Aptitude Test (SAT) and Graduate Record
Examination (GRE) are using item response theory (IRT) to generate score scales for alternative forms
~\citep{an2014item}. This test equating process can be divided into two steps, linking and equating.
Linking refers to how to estimate the IRT parameters of students and questions under the model; and
equating refers to how to adjust the raw grade of the students to adapt to different overall difficulty
levels in different version of the exam, e.g.~\citet{leeIRTLinkingEquating2018a}. One of the most popular
test equating processes is IRT true-score equating with nonequivalent-groups anchor test (NEAT) design. In the
NEAT design, there are two test forms given to two population of students, where a set of common questions is
contained in both forms. Linking performs by putting the estimated parameter of the common items onto the same
scale through a linear transformation, since any linear transformation gives the same probability under the
IRT model. Equating performs by taking the estimated ability of the student from the second form and compute
the expected number of accurate answers in the first form as the adjusted grade. Since these large-scale
standardized tests have a large population of students for each variant of the exam, the above test equating
process works well. 
Our methods can be viewed as adapting the statistical framework of linking and equating to the administration of a single exam for a small population of students.
In our randomized exam setting with small scale, however, every student receives a
different form of the exam, thus it is almost impossible to estimate the parameters for every form separately or to decide an anchor set of question and do the same linking. Our algorithm uses the 
concurrent linking that estimates all parameters at the same time based on the information
in all forms. As for equating, we use a similar method of true-score equating, but 
compute on the whole question bank instead of one specific form.

In the problem of fair allocation of indivisible items, Best-of-Both-Worlds (BoBW) fairness mechanisms (e.g., \citet{azizSimultaneouslyAchievingExante2020, freemanBestBothWorlds2020, babaioffBestofBothWorldsFairShareAllocations2022}) try to provide both ex-ante fairness and ex-post fairness to agents. An ex-ante fair mechanism is easy to be found. For example, giving all items to one random agent guarantees that every agent receives a $\frac{1}{n}$ fraction of the total value in expectation (ex-ante proportionality). However, such a mechanism is clearly not ex-post fair. Likewise, simple averaging gives every student an unbiased grade ex-ante, but neglects the different overall difficulty among students ex-post. We propose another grading rule that evaluates the difficulties of the questions and adjusts the grades according to them, which achieves better ex-post fairness of the students.
\section{Model}\label{sec:model}
Consider a set of students $S$ and a bank of questions $Q$. A merit
vector $\boldsymbol{u}$ is used to describe the key property of the
students and questions. Specifically, for any student $i\in S$, $u_i$
represents the ability of the student; for any question $j\in Q$,
$u_j$ represents the difficulty of the question. We put them in the same
vector for convenience. The merit vector is
unknown when the exam is designed. Denote $w_{ij}$ as the outcome of
the answering process. Then $w_{ij}$s are independent Bernoulli
random variables, where $w_{ij}=1$ represents a correct answer,
$w_{ij}=0$ represents an incorrect answer, and
\[\Pr[w_{ij}=1]=1-\Pr[w_{ij}=0]=\frac{\exp(u_i)}{\exp(u_i)+\exp(u_j)}=f(u_i-u_j),\]
where $f(x)=\frac{1}{1+\exp(-x)}$. The goal of the exam design is to assign a small number of questions to each student (task assignment graph), and based on the exam result (exam result graph), give each student a grade (grading rule) that accurately estimates her performance over the whole question bank (benchmark). We give a formal description of the task assignment graph, exam result graph, benchmark, and grading rule below.

\begin{definition}[Task Assignment Graph]
The task assignment graph $G=(S\cup Q, E)$ is an undirected bipartite
graph, where the left part of the vertices represents the students and
the right part represents the questions, and an edge between $i\in S$
and $j\in Q$ exists if and only if the instructor decides to assign
question $j$ to student $i$.
\end{definition}

\begin{definition}[Exam Result Graph]
The exam result graph $G'=(S\cup Q, E')$ is a directed bipartite graph
constructed from the task assignment graph $G$. All directed edges are
between students and questions. For any edge $(i,j)\in G$ in the
task assignment graph, where $i\in S$ and $j\in Q$, if student $i$
answers question $j$ correctly in the exam, i.e., we observe that
$w_{ij}=1$, there is an edge $i\to j$ in $G'$; if the answer is
incorrect, i.e., we observe that $w_{ij}=0$, there is an edge $j\to i$
in $G'$. For other student-question pairs that do not occur in the
task assignment graph $G$, there is also no edge between them in the
exam result graph $G'$.
\end{definition}

% Before the exam is taken, the instructor needs to decide the task assignment graph. There are some heuristic standards for generating such graphs. The first one is to assign only a small part of questions to a student, i.e., upper bounding the maximum degree among the students. It lowers the burden of the students. The second one is try to assign different questions to different students, i.e., increasing the edge expansion of $G$. It reduces the communication between students and provides the instructor with an overview of who are the better students and which are the more difficult questions. The third one is to be careful about the size of question bank. For example, given a question bank with infinite size, following first two standards would make the exam result graph disconnected. Then we lose the most helpful overlapping information. One way to solve this problem is to sample some questions from the bank.

To evaluate different exam designs and grading rules, we propose the following benchmark.

\begin{definition}[Benchmark]
In an ideal case where we know the distribution over the outcome of the answering processes $w_{ij}$s, the instructor would measure the students' performance by their expected accuracy on a random question in the bank. Formally, the benchmark for any student $i$'s grade is
\begin{equation}\label{def:benchmark}
    \mathrm{opt}_i=\E_{j\sim\mathcal U(Q)}[w_{ij}]=\frac{1}{|Q|}\sum_{j\in Q}f(u_i-u_j).
\end{equation}
\end{definition}

The benchmark is an ideal way to grade the student if the instructor has complete information on all answering processes. On the other hand, when the instructor only observes one sample of each $w_{ij}$ involved in the exam, we will use a grading rule to grade the students.

\begin{definition}[Grading Rule]
In an exam, the instructor gives a grade for each student based on the exam result graph. A grading rule is a mapping $g\colon G'\to \mathbb{R}^{S}$ from the exam result graph to the grades for each student.
\end{definition}

One interpretation of the grade is as an estimation of the benchmark, i.e., students' expected accuracy on a random question in the bank, which combines the two important criteria of fairness and accuracy. To evaluate the exam design, we compare the performance of the grading rule to the benchmark and aggregate the error among all students. Specifically, there are three stages of the exam design, before the randomization of the task assignment graph, after the randomization of the task assignment graph and before the student answering process, and after the student answering process.
In each stage, we might care about the maximum or average unfairness among students.

\begin{definition}[Ex-ante Bias]
For a given algorithm $\mathrm{alg}$, the ex-ante bias for student $i$ is defined as the mean square error of the algorithm's expected performance compared to the benchmark, over a random family $\mathcal G$ of task assignment graphs, i.e., $\left(\E_{G\sim\mathcal G}\E_{w}[\mathrm{alg}_i]-\mathrm{opt}_i\right)^2$.
\end{definition}
\begin{definition}[Ex-post Bias]
For a given algorithm $\mathrm{alg}$ and a fixed task assignment graph $G$, the ex-post bias for student $i$ is defined as the mean square error of the algorithm's expected performance compared to the benchmark on $G$, i.e., $\left(\E_{w}[\mathrm{alg}_i]-\mathrm{opt}_i\right)^2$.
\label{definition:ex-post bias}
\end{definition}
\begin{definition}[Ex-post Error]
For a given algorithm $\mathrm{alg}$, a fixed task assignment graph $G$, and a fixed realization of the student answering process $w$, the ex-post error for student $i$ is defined as the mean square error of the algorithm's performance compared to the benchmark on $G$ and $w$, i.e., $\left(\mathrm{alg}_i-\mathrm{opt}_i\right)^2$.
\label{definition:ex-post error}
\end{definition}

The difference between ex-ante bias, ex-post bias, and ex-post error is that ex-ante bias takes
expectation over both random graphs and the noisy answering process, ex-post bias takes expectation over the noisy answering process, while ex-post error directly measures the error. Thus ex-post error is harder than ex-post bias which is harder than ex-ante bias to achieve.

\begin{example}[Simple Averaging]
Simple averaging is a commonly used grading rule in exams. It calculates the average accuracy on the questions the student receives. Formally, given a exam result graph $G'$, the simple averaging grades student $i$ by
\begin{equation}
    \mathrm{avg}_i=\frac{\deg^{+}_i}{\deg^{-}_i+\deg^{+}_i}=\frac{\sum_{j}1_{(i,j)\in E'}}{\sum_{j}1_{(i,j)\in E}},
\end{equation}
where $\deg^{+}$ and $\deg^{-}$ represents the outdegree and indegree of the vertex in $G'$, respectively.
\label{example:Avg}
\end{example}
\begin{theorem}
The simple averaging is ex-ante fair over any family of bipartite graphs $\mathcal G$ that is symmetric with respect to the questions, i.e., its ex-ante bias is 0.
\end{theorem}
\begin{proof}
\begin{equation}
\begin{aligned}
\forall i,~\E_{G\sim \mathcal G}\E_{w}\left[\mathrm{avg_i}\right]
&=\E_{G\sim \mathcal G}\E_{w}\left[\frac{\sum_{j}1_{(i,j)\in E'}}{\sum_{j}1_{(i,j)\in E}}\right]=\E_{G\sim \mathcal G}\E_{w}\left[\frac{\sum_{j}w_{ij}1_{(i,j)\in E}}{\sum_{j}1_{(i,j)\in E}}\right]\\
&=\E_{G\sim \mathcal G}\left[\frac{\sum_{j}\E[w_{ij}]1_{(i,j)\in E}}{\sum_{j}1_{(i,j)\in E}}\right]=\sum_j\E[w_{ij}]\E_{G\sim \mathcal G}\left[\frac{1_{(i,j)\in E}}{\sum_{j}1_{(i,j)\in E}}\right]=\mathrm{opt_i}
% &=\sum_{s=1}^n\E_{G\sim \mathcal G}\left[\frac{\sum_{j}\E[w_{ij}]1_{(i,j)\in E}}{s}\middle|\sum_{j}1_{(i,j)\in E}=s\right]\Pr_{G\sim\mathcal G}\left[\sum_{j}1_{(i,j)\in E}=s\right]\\
%&=\mathrm{opt_i}\sum_{s=1}^n\Pr_{G\sim\mathcal G}\left[\sum_{j}1_{(i,j)\in E}=s\right]=
\end{aligned}
\end{equation}
\end{proof}
In other words, simple averaging can be seen as an ex-ante unbiased estimator of the benchmark.
However, ex-post, i.e., on one specific task assignment graph, simple averaging is unfair. Intuitively, some unlucky students might be assigned harder questions and receive a significantly lower average grade than the benchmark, and the opposite happens to some lucky students. We will visualize this phenomenon in \Cref{fig:fixed} in \Cref{subsec:visual}.

Based on the above definitions, we now formalize the procedure and goal of the exam grading problem.
\begin{enumerate}[i.]
\setlength{\itemsep}{0pt}
\setlength{\parsep}{0pt}
    \item The instructor chooses a task assignment graph $G$.
    \item The students receive questions according to $G$ and give their answer sheet back, thus the instructor receives the exam result graph $G'$.
    \item The instructor uses a grading rule $g$ to grade the students based on $G'$.
    \item We want the grade $g(G')$ to have a small maximum (average) ex-post bias or ex-post error.
\end{enumerate}

\section{Method}
In this section, we propose our method for the exam grading
problem. According to our formalization of the problem, any method
contains two parts: generating the task assignment graph $G$, and
choosing the grading rule $g$. We describe each of them respectively.
\subsection{Task Assignment Graph}
For simplicity, we assume both the student set $S$ and the question
set $Q$ is finite. To generate the task assignment graph, we
sample $m$ different questions u.a.r.\ from the question bank, and
independently assign each student $d$ different questions u.a.r.\ from
those $m$ questions.

\begin{algorithm}
\caption{Task assignment graph generation}\label{alg:graph}
\begin{algorithmic}
\Require finite sets $S$ and $Q$, question sample size $1\leq m\leq |Q|$, degree constraint $1\leq d\leq m$
\Ensure a task assignment graph $G=(S\cup Q,E)$
\State $\tilde{Q}\gets$ a set of $m$ questions sampled u.a.r. without replacement from $Q$
\ForAll {$i\in S$}
    \State $J\gets$ a set of $d$ questions sampled u.a.r. without replacement from $\tilde{Q}$
    \State $E \gets E\cup\{(i,j)|j\in J\}$
\EndFor
\end{algorithmic}
\end{algorithm}

\subsection{Grading Rule}\label{sec:method:grading rule}
Recall that a grading rule maps from an exam result graph $G'$ to a
vector of probabilities. In contrast with simple averaging which only
considers the local information (the in-degrees and out-degrees of the
students), we use structural information of the exam result graph for
analysis. Our grading rule is an aggregation of a prediction matrix
$h\in [0,1]^{S\times Q}$, where $h_{ij}$ represents the algorithm's
prediction on the probability that student $i$ answers correctly
question $j$. The grade for student $i$ will be the average of
$h_{ij}s$ over all $j\in Q$,
i.e. $\mathrm{alg}_i=\frac{1}{|Q|}\sum_{j\in Q}h_{ij}$. We use
$u\rightsquigarrow v$ to represent the existence of a directed path in
$G'$ that starts with $u$ and ends with $v$, and
$u\centernot\rightsquigarrow v$ for nonexistence. The algorithm
classifies the elements $h_{ij}$s into four cases: existing edge
$(i,j)\in E$, same component $i\rightsquigarrow j\land
j\rightsquigarrow i$, comparable components $i\rightsquigarrow j\oplus
j\rightsquigarrow i$, and incomparable components
$i\centernot\rightsquigarrow j\land j\centernot\rightsquigarrow i$.

\paragraph{Existing Edge}
For $(i,j)\in E$, we observe $w_{ij}$ from the exam result graph $G'$, hence $h_{ij}=w_{ij}$.

\paragraph{Same Component}
For student $i\in S$ and question $j\in Q$ satisfy $i\rightsquigarrow j\land j\rightsquigarrow i$, they are in the same strongly connected component in $G'$. We make all predictions in the component simultaneously, by inferring the student abilities and question difficulties from the structure of the component. Formally, denote $V'$ as the vertex set of the component. From \Cref{thm:equiv}, the strong connectivity guarantees the existence of the maximum likelihood estimators (MLEs) $\boldsymbol{u^*}\in\mathbb R^{V'}$. We can use a minorization–maximization algorithm from~\citet{hunterMMAlgorithmsGeneralized2004} to calculate the MLEs and set $h_{ij}=f(u_i^*-u_j^*)$ for any missing edge $(i,j)$ between students and questions inside this component.

\paragraph{Comparable Components} W.l.o.g., we assume $i\rightsquigarrow j$ and $j\centernot\rightsquigarrow i$, thus every directed path between those two vertices starts with the student and ends with the question, showing strong evidence of a correct answer. In other words, considering the strongly connected components they belong to, the component that contains the student has a ``higher level'' in the condensation graph of $G'$ and can reach the component that contains the question, i.e., they belong to comparable components in the condensation graph. In this case, we set $h_{ij}=1$. Similarly, if $j\rightsquigarrow i$ and $i\centernot\rightsquigarrow j$, we set $h_{ij}=0$

\paragraph{Incomparable Components} For a student $i$ and question $j$ that satisfy $i\centernot\rightsquigarrow j\land j\centernot\rightsquigarrow i$, i.e., they lie in incomparable components, we use the average of the predictions in the above three cases as the prediction for $h_{ij}$.

\begin{algorithm}
\caption{Grade Generation}\label{alg:grade}
\begin{algorithmic}
\Require an exam result graph $G'(S\cup Q, E')$
\Ensure a grade vector $g\in[0,1]^{S}$ for students
\State From the exam result graph $G'$, we can get the task assignment graph $G(S\cap Q, E)$.
\ForAll {$(i,j)\in E$}
\Comment{Case 1: Existing Edge}
    \State $h_{ij}\gets w_{ij}$.
\EndFor
\ForAll {Strongly Connected Component $\tilde G(\tilde S\cup\tilde Q,\tilde E)$}
\Comment{Case 2: Same Component}
    \State $\boldsymbol{u^*}\gets$ the MLEs of the merit parameters of $\tilde S\cup\tilde Q$.
    \ForAll{$(i,j)\in (\tilde S\times\tilde Q)\setminus \tilde E$}
        \State $h_{ij}\gets f(u^*_i-u_j^*)$.
    \EndFor
\EndFor
\ForAll{$(i,j)\in (S\times Q)\setminus E\land i\rightsquigarrow j\land j\centernot\rightsquigarrow i$}
\Comment{Case 3: Comparable Component}
    \State $h_{ij}\gets 1$.
\EndFor
\ForAll{$(i,j)\in (S\times Q)\setminus E \land j\rightsquigarrow i\land i\centernot\rightsquigarrow j$}
    \State $h_{ij}\gets 0$.
\EndFor
\ForAll{$(i,j)\in (S\times Q)\setminus E\land i\centernot\rightsquigarrow j\land j\centernot\rightsquigarrow i$}
\Comment{Case 4: Incomparable Component}
    \State $h_{ij}\gets$ the average of the existing $h_{i\cdot}$ in previous steps.
\EndFor
\ForAll{$i\in S$}
\Comment{Grade Aggregation}
\State $g_i\gets$ the average of $h_{ij}$ for all $j$s.
\EndFor
\end{algorithmic}
\end{algorithm}

\section{Theory}\label{sec:theory}
In this section, we show several properties of our algorithm.
Recall that the Bradley-Terry-Luce model describes the outcome of pairwise comparisons as follows.
In a comparison between subject $i$ and subject $j$, subject $i$ beats subject $j$ with probability
\[p_{ij}=\frac{\exp(u_i)}{\exp(u_i)+\exp(u_j)}=f(u_i-u_j),\]
where $\boldsymbol{u}=(u_1\dots,u_{n+m})$ represents the merit parameters of $n+m$ subjects and $f(x)=\frac{1}{1+\exp(-x)}$.
We consider the Bradley-Terry-Luce model under a family of random bipartite task assignment graphs $\mathcal{B}(n,m,d_{n,m})$.
Specifically, a task assignment graph $G(L\cup R, E)$ with $n$ vertices in $L$ and $m$ vertices in $R$, where $n\leq m$, 
is constructed by linking $d_{n,m}$ different random vertices in $R$ to each left vertex in $L$,
i.e., $L$ is regular but $R$ is not.

Given a task assignment graph $G$, denote $A$ as its adjacency matrix.
For any two subjects $i$ and $j$, the number of comparisons between them follows $A_{ij}\in\{0,1\}$.
We define $A'_{ij}$ as the number of times that subject $i$ beats subject $j$, thus $A'_{ij}+A'_{ji}=A_{ij}=A_{ji}$.
In other words, $A'$ is the adjacency matrix of the exam result graph $G'$.
Based on the observation of $G'$, the log-likelihood function is
\begin{equation}
\mathcal L(\boldsymbol{u})=\sum_{1\leq i\neq j\leq n+m}A'_{ij}\log p_{ij}=\sum_{1\leq i\neq j\leq n+m}A'_{ij}\log f(u_i-u_j).
\label{eq:likelihood}
\end{equation}

Denote $\boldsymbol{u}^*=(u_1^*,u_1^*,\dots,u_{n+m}^*)$ as the maximum likelihood estimators (MLEs) of $\boldsymbol{u}$.
Since $\mathcal{L}$ is additive invariant,  w.l.o.g.\ we assume $u_1=0$ and set $u^*_1=0$.
Since $(\log f(x))'=1-f(x)$ the likelihood equation can be simplified to
\begin{equation}\label{eq:likelihood:foc}
\sum_{j=1}^{n+m} A'_{ij}=\sum_{j=1}^{n+m} A_{ij}f(u_i^*-u_j^*),\forall~i.
\end{equation}

\subsection{Existence and Uniqueness of the MLEs}
\citet{zermeloBerechnungTurnierErgebnisseAls1929c} and \citet{fordSolutionRankingProblem1957} gave a necessary and sufficient condition
for the existence and uniqueness of the MLEs in \eqref{eq:likelihood:foc}.
\paragraph{Condition A.}
For every two nonempty sets that form a partition of the subjects, a subject in one set has beaten a subject in the other set at least once.\\

To provide an intuitive understanding of Condition A, we show its equivalence to the strong connectivity of the exam result graph $G'$. Then we state our theorem on
when Condition A holds.

\begin{theorem}
Condition A holds if and only if the exam result graph $G'$ is strongly connected.
\label{thm:equiv}
\end{theorem}
\begin{proof}
Condition A says that for any partition $(V_1, V_2)$ of the vertices $L\cup R$,
there exists an edge from $V_1$ to $V_2$ and also an edge from $V_2$ to $V_1$.
If $G'$ is strongly connected, Condition A directly holds by the definition of strong connectivity.
Otherwise, if $G'$ is not strongly connected, the condensation of $G'$ contains at least two SCCs.
We pick one strongly connected component with no indegree as $V_1$ and the remaining vertices as $V_2$,
then there is no edge from $V_2$ to $V_1$, i.e., Condition A fails.
\end{proof}

\begin{theorem}[Existence and Uniqueness of MLEs]
If
\begin{equation}
\frac{\exp(\alpha_{n,m})(n+m)\log(n+m)}{nd_{n,m}}\rightarrow 0\quad (n,m\rightarrow\infty),    
\label{condition:connectivity}
\end{equation}
where $\alpha_{n,m}=\max_{1\leq i,j\leq n+m} u_i-u_j$ is the largest difference between all possible pairs of merits,
then $\Pr\left[\text{Condition A is satisfied}\right]\rightarrow 1\quad (n,m\rightarrow\infty).$
\label{thm:connectivity}
\end{theorem}

To prove \Cref{thm:connectivity}, we analyze the edge expansion property (\Cref{thm:edge-expansion}) of the task assignment graph $G$
and take a union bound on all valid subsets to bound the probability
that $G'$ fails Condition A.

\begin{lemma}[Edge Expansion]
Under condition \eqref{condition:connectivity},
\[\Pr\left[\forall S\subset V,~\text{s.t.}~|S|\leq \frac{n+m}{2},
\quad \frac{|\partial S|}{|S|} > \frac{nd_{n,m}}{2(n+m)}\right]\rightarrow 1\quad (n,m\rightarrow\infty),\]
where $\partial S=\{(u,v)\in E:u\in S, v\in V\setminus S\}$ for the task assignment graph $G(V,E)$.
\label{thm:edge-expansion}
\end{lemma}

\begin{proof}
Consider any subset of vertices $S$ with size $r\leq \frac{n+m}{2}$.
Denote $X=S\cap L, Y=S\cap R, |X|=x$, thus $|Y|=r-x, |L\setminus X|=n-x, |R\setminus Y|=m+x-r$.
$\partial S$ is a random variable that can be expressed as
\[|\partial S|=\sum_{u\in X}\sum_{v\in R\setminus Y}A_{uv}+\sum_{u\in L\setminus X}\sum_{v\in Y}A_{uv},\]
where $A$ is the adjacency matrix of the task assignment graph $G$.
Recall that the task assignment graph $G$ is generated by
linking $d_{n,m}$ random different vertices in $R$ to each vertex in $L$.
Thus for different $u_1\neq u_2\in L$, $A_{u_1\cdot}$ is independent with $A_{u_2\cdot}$,
while for a fixed $u\in L$, $A_{u\cdot}$ is chosen randomly without replacement.
Chernoff bound applies under such conditions, i.e.,
\[\Pr\left[|\partial S|\leq \frac{1}{2}\E\left[|\partial S|\right]\right]\leq \exp\left(-\frac{\E[|\partial S|]}{8}\right).\]
Then we lower bound $\E[|\partial S|]$ by
\begin{equation*}
\begin{aligned}
\E[|\partial S|]
&=\frac{d_{n,m}}{m}(|X||R\setminus Y|+|L\setminus X||Y|)\\
&=\frac{d_{n,m}}{m}\left(2x^2+(m-n-2r)x+nr\right).
\end{aligned}
\end{equation*}
For the case where $m-n-2r\leq 0$, i.e., $r\geq \frac{m-n}{2}$, we have 
\begin{equation*}
\begin{aligned}
\E[|\partial S|]
&=\frac{d_{n,m}}{m}\left(2x^2+(m-n-2r)x+nr\right)\\
&\geq \frac{d_{n,m}}{m}\left(-\frac{(m-n-2r)^2}{8}+nr\right)
    =\frac{d_{n,m}r}{m}\left(-\frac{1}{2}r-\frac{1}{8}\frac{(m-n)^2}{r}+\frac{1}{2}(n+m)\right)\\
&\geq \frac{d_{n,m}r}{m}\left(-\frac{n+m}{4}-\frac{1}{4}\frac{(m-n)^2}{n+m}+\frac{1}{2}(n+m)\right)
    =\frac{nd_{n,m}r}{n+m}
\end{aligned}
\end{equation*}
For the case where $m-n-2r > 0$, i.e., $r < \frac{m-n}{2}$, we have
\begin{equation*}
\begin{aligned}
\E[|\partial S|]=\frac{d_{n,m}}{m}\left(2x^2+(m-n-2r)x+nr\right)
\geq \frac{nd_{n,m}r}{m}\geq \frac{nd_{n,m}r}{n+m}.
\end{aligned}
\end{equation*}
Thus for any fixed set $S$ with size $r\leq \frac{n+m}{2}$,
\[\Pr\left[|\partial S|\leq \frac{d_{n,m}nr}{2(n+m)}\right]
\leq \Pr\left[|\partial S|\leq \frac{1}{2}\E\left[|\partial S|\right]\right]
\leq \exp\left(-\frac{\E[|\partial S|]}{8}\right)\leq \exp\left(-\frac{nd_{n,m}r}{8(n+m)}\right).\]

Finally, by union bound,
\begin{equation*}
\begin{aligned}
    &\Pr\left[\forall S\subset V,~\text{s.t.}~|S|\leq n,
        \quad \frac{|\partial S|}{|S|} > \frac{nd_{n,m}}{2(n+m)}\right]\\
    =&1-\Pr\left[\exists S\subset V,~\text{s.t.}~|S|\leq n,
        \quad \frac{|\partial S|}{|S|} \geq  \frac{nd_{n,m}}{2(n+m)}\right]\\
    \geq& 1-\sum_{r=1}^{(n+m)/2}\binom{n+m}{r}\exp\left(-\frac{nd_{n,m}r}{8(n+m)}\right)\\
    \geq& 1-\sum_{r=1}^{(n+m)/2}\exp\left(-\frac{nd_{n,m}r}{8(n+m)}+r\log(n+m)\right)\\
    \geq& 1-\sum_{r=1}^{(n+m)/2}\exp\left(-\frac{nd_{n,m}r}{16(n+m)}\right)\\
    \geq& 1-\exp\left(-\frac{nd_{n,m}}{16(n+m)}+\log (n+m)\right)\\
    \geq& 1-\exp\left(-\frac{nd_{n,m}}{32(n+m)}\right)\\
\end{aligned}
\end{equation*}
The third-to-last inequality and the last inequality hold when $d_{n,m}>\frac{32(n+m)\log (n+m)}{n}$.
Note that condition $\eqref{condition:connectivity}$ implies
$\frac{(n+m)\log (n+m)}{nd_{n,m}}\rightarrow 0\quad (n,m\rightarrow\infty)$ since $\alpha_{n,m}\geq 0$.
Thus for large enough $n$ and $m$,
\begin{equation*}
    \Pr\left[\forall S\subset V,~\text{s.t.}~|S|\leq n,
        \quad \frac{|\partial S|}{|S|} > \frac{nd_{n,m}}{2(n+m)}\right]
    \geq 1-\exp\left(-\frac{nd_{n,m}}{32(n+m)}\right)\rightarrow 1\quad (n,m\rightarrow\infty).
\end{equation*}
\end{proof}

\begin{proof}[Proof of Theorem \ref{thm:connectivity}]
For an edge between vertex $i$ and $j$ in the task assignment graph $G$, i.e. $A_{ij}=1$,
the corresponding directed edge in the exam result graph $G'$ goes from $i$ to $j$ with probability 
\[\Pr[A'_{ij}=1]=f(u_i-u_j)\leq \max_{1\leq i,j\leq n+m} f(u_i-u_j)
\leq \frac{1}{1+\exp(-\alpha_{n,m})}\leq 2^{-\exp(-\alpha_{n,m})}.\]
By Lemma \ref{thm:edge-expansion}, under condition \eqref{condition:connectivity},
\[\Pr\left[\forall S\subset V,~\text{s.t.}~|S|\leq n,
\quad \frac{|\partial S|}{|S|} > \frac{nd_{n,m}}{2(n+m)}\right]
\rightarrow 1\quad (n,m\rightarrow\infty).\]
Now consider any subset of vertices $S\subset V~\text{s.t.}~|S|=r\leq \frac{n+m}{2}$.
The probability that all edges between $S$ and $V\setminus S$ go in the same direction in $G'$ is
no more than $2\left(2^{-\exp(-\alpha_{n,m})}\right)^{\frac{nd_{n,m}}{2(n+m)}}$. Thus by union bound, the probability that Condition A holds is at least
\begin{equation*}
\begin{aligned}
    &1-2\sum_{1\leq r\leq (n+m)/2}\binom{n+m}{r}\left(2^{-\exp(-\alpha_{n,m})\frac{nd_{n,m}}{2(n+m)}}\right)\\
    \geq &  1-2\left(\sum_{0\leq r\leq n+m}\binom{n+m}{r}\left(2^{-\exp(-\alpha_{n,m})\frac{nd_{n,m}}{2(n+m)}}\right)-1\right)\\
    \geq & 1-2\left(\left(1+\left(2^{-\exp(-\alpha_{n,m})\frac{nd_{n,m}}{2(n+m)}}\right)\right)^{n+m}-1\right),\\
\end{aligned}
\end{equation*}
which converges to 1 when $n,m\rightarrow \infty$ under condition $\eqref{condition:connectivity}$.
\end{proof}

\subsection{Uniform Consistency of the MLEs}
Based on condition \eqref{condition:connectivity},
\Cref{thm:connectivity} shows the existence and uniqueness of the MLEs.
In this part, we give an outline of the proof
for the uniform consistency of the MLEs (\Cref{thm:consistency}).

\begin{theorem}[Uniform Consistency of MLEs]
If
\begin{equation}
    \exp\left(2(\alpha_{n,m}+1)\right)
        \Delta_{n,m}\rightarrow 0\quad(n,m\rightarrow\infty),
    \label{condition:consistency}
\end{equation}
where $\Delta_{n,m}= \sqrt{\frac{m\log^3 (n+m)}
    {nd_{n,m}\log^2 (\frac{n}{m}d_{n,m})}}$,
then the MLEs are uniformly consistent, i.e.,
$\|\boldsymbol{u}^*-\boldsymbol{u}\|_\infty\stackrel{\mathbb{P}}{\longrightarrow} 0$.
\label{thm:consistency}
\end{theorem}

Denote $\varepsilon_i=u^*_i-u_i$ as the estimation error of the maximum likelihood estimators.
Since we assume $u_1=0$ and set $u_1^*=0$, we have $\varepsilon_1=u^*_1-u_1=0$.
Consider the two subjects with the most negative estimation error and the most positive estimation error $\lowelement=\mathop{\arg\min}\limits_i \varepsilon_i\leq \varepsilon_1=0$, $\highelement = \mathop{\arg\max}\limits_i \varepsilon_i\geq \varepsilon_1=0$,
and their corresponding error $\loweps=\min\limits_i \varepsilon_i$, $\higheps = \max\limits_i \varepsilon_i,$
then we have $\|\boldsymbol{u}^*-\boldsymbol{u}\|_\infty=\max\{-\loweps,\higheps\}\leq \higheps-\loweps$. The goal is to identify a specific number $D$, such that
more than half $\varepsilon_i$s are at most $\loweps+D$,
and more than half $\varepsilon_i$s are at least $\higheps-D$.
Then at least one subject is on both sides,
thus $\higheps-\loweps$ is bounded by $2D$.

To identify $D$, we check a sequence of increasing numbers $\{D_k\}_{k=0}^{K_{n,m}}$,
and the two corresponding growing sets
$\{\lowset_k\}_{k=0}^{K_{n,m}}$ and $\{\highset_k\}_{k=0}^{K_{n,m}}$
that contains the subjects with estimation errors $D_k$-close to $\loweps$ and $\higheps$ respectively.
Under careful choice of $K_{n,m}$ and $\{D_k\}_{k=0}^{K_{n,m}}$,
we will show that $\lowset_{K_{n,m}}$ and $\highset_{K_{n,m}}$ both contain more than half subjects.

The main difficulty is showing the growth of
$\{\lowset_k\}_{k=0}^{K_{n,m}}$ and $\{\highset_k\}_{k=0}^{K_{n,m}}$.
We prove this by considering the local growth of the sets, i.e.,
$N(\lowset_k)\cap \lowset_{k+1}$ and $N(\highset_k)\cap \highset_{k+1}$.
By symmetry, we only consider $\lowset_k$.
\Cref{thm:vertex-expansion} analyzes the generation of the random task assignment graphs
and shows a vertex expansion property that describes the growth of the neighborhoods $N(\lowset_k)$.
\Cref{thm:local-property} starts with any vertex $i$ in $\lowset_k$,
analyzes the first order equations of the MLE to exclude the vertices that
are in the neighborhoods $N(\{i\})$ and but are not in $\lowset_{k+1}$,
and gives a lower bound on the size of $N(\{i\})\cap \highset_{k+1}$.
Finally, we jointly consider all vertices in $\lowset_k$
and provide a lower bound on the size of $N(\lowset_k)\cap \lowset_{k+1}$, which shows the growth rate of
$\lowset_k$ and finishes the proof.

\paragraph{Definition of Notations}
\begin{itemize}
    \item $K_{n,m} = 2\left\lceil\frac{\log n}{\log (\frac{n}{m}d_{n,m})}-1\right\rceil$ is the number of steps of the growth.
    \item $c_{n,m} = \frac{\exp(-(\alpha_{n,m}+1))}{4}$ is a lower bound on $f'(x)$ for $|x|\leq \alpha_{n,m}+1$.
    \item $q_{n,m}=\frac{c_{n,m}\log (\frac{n}{m}d_{n,m})}{5\log n}$ is a lower bound on the local growth rate $\frac{|N(\{i\})\cap \lowset_{k+1}|}{|N(\{i\})|}$ of vertex $i\in\lowset_{k}$.
    \item $z_{n,m}=\sqrt{\frac{32m\log (n+m)}{nd_{n,m}}}$ is the deviation used in the Chernoff bound.
    \item 
The sequence of numbers $\{D_k\}_{k=0}^{K_{n,m}}$ is set to be
\begin{align*}
&D_k=\frac{4k}{c_{n,m}}\sqrt{\frac{2m\log (n+m)}{(1-z_{n,m})nd_{n,m}}}
    \quad \text{for}~k=0,1,\dots,K_{n,m}-1,\\
&D_{K_{n,m}}=\frac{80K_{n,m}}{c_{n,m}^2}\sqrt{\frac{2m\log (n+m)}{(1-z_{n,m})nd_{n,m}}}.
\end{align*}
    \item The two growing sets $\{\lowset_k\}_{k=0}^{K_{n,m}}$ and $\{\highset_k\}_{k=0}^{K_{n,m}}$
which contains the subjects with estimation error $D_k$-close to
$\loweps$ and $\higheps$ respectively are defined as
\begin{align*}
&\lowset_k=\{j:\varepsilon_j-\loweps\leq D_k\},\\
&\highset_k=\{j:\higheps-\varepsilon_j\leq D_k\}.
\end{align*}

\end{itemize}

\begin{lemma}[Vertex Expansion]
Regarding the task assignment graph $G(L\cup R,E)\sim\mathcal B(n,m,d_{n,m})$,
for a fixed subset of left vertices $X\subset L$ with $|X|\leq \frac{n}{2}$,
w.p.\ $1-(n+m)^{-4|X|}$ it holds that
\begin{itemize}
    \item If $1\leq |X|<m/d_{n,m}$,
        $\frac{|N(X)|}{|X|}>(1-z_{n,m})\left(1-\frac{d_{n,m}|X|}{m}\right)d_{n,m};$
    \item  If $|X|\geq m/d_{n,m}$,
        $\frac{|N(X)|}{m}>1-z_{n,m}-e^{-1}.$
\end{itemize}
For a fixed subset of right vertices $Y\subset R$ with $|Y|\leq \frac{m}{2}$,
w.p.\ $1-(n+m)^{-4|Y|}$ it holds that
\begin{itemize}
    \item If $1\leq |Y|<m/d_{n,m}$,
    $\frac{|N(Y)|}{|Y|}>(1-z_{n,m})\left(1-\frac{d_{n,m}|Y|}{m}\right)\frac{nd_{n,m}}{m};$
    \item  If $|Y|\geq m/d_{n,m}$,
    $\frac{|N(Y)|}{n}>1-z_{n,m}-e^{-1}.$
\end{itemize}
In above inequalities, $z_{n,m}=\sqrt{\frac{32m\log (n+m)}{nd_{n,m}}}$ as previously defined.
\label{thm:vertex-expansion}
\end{lemma}

\begin{proof}
Before proving the vertex expansion property of
the task assignment graph $\mathcal{B}(n,m,d_{n,m})$,
we first bound the vertex degree by Chernoff bound and union bound,
\begin{equation}\label{thm:degree}
\forall~i\in R,\quad \Pr\left[(1-z_{n,m})\frac{nd_{n,m}}{m}\leq
\left|N(\{i\})\right|\leq (1+z_{n,m})\frac{nd_{n,m}}{m}\right]\geq  1-(n+m)^{-4},
\end{equation}
where $z_{n,m}$ is defined above as $z_{n,m}=\sqrt{\frac{32m\log (n+m)}{nd_{n,m}}}\rightarrow 0$
when $n,m\rightarrow\infty$ under condition \eqref{condition:connectivity}.

We need to define another family of random bipartite graph $\tilde{\mathcal B}$.
Each graph in $\tilde{\mathcal B}(n,m,d_{n,m})$ contains $n$ vertices in the left part,
$m$ vertices in the right part,
and assigns $d_{n,m}$ random neighbors to each vertex in the left part (multi-edges are allowed).
For any $X\subset L$, it's easy to see that $|N(X)|$ in $G\sim \mathcal B(n,m,d_{n,m})$
stochastically dominates $|N(X)|$ in $G\sim\tilde{\mathcal B}(n,m,d_{n,m})$.
Thus it's sufficient to prove the theorem under $\tilde {\mathcal B}(n,m,d_{n,m})$.
On the other hand, counting $|N(X)|$ under $\tilde{\mathcal B}(n,m,d_{n,m})$ is
the same random process as counting the number of non-empty bins
after independently throwing $d_{n,m}|X|$ balls u.a.r.\ into $m$ bins.
By linearity of expectation over every bin, we know
\[\E[|N(X)|]=m\left(1-\left(1-\frac{1}{m}\right)^{d_{n,m}|X|}\right).\]
We need several lower bounds of $\E[|N(X)|]$ here. With the fact of
\[\frac{x}{2}\leq 1-\exp(-x)\leq x,\quad \forall~0\leq x<1,\]
we have
\begin{equation*}
\begin{aligned}
\E[|N(X)|]=m\left(1-\left(1-\frac{1}{m}\right)^{d_{n,m}|X|}\right)
\geq m\left(1-\exp\left(-\frac{d_{n,m}|X|}{m}\right)\right)\geq \frac{d_{n,m}|X|}{2}.
\end{aligned}
\end{equation*}
Therefore, using Azuma's inequality, we can lower bound $|N(X)|$, i.e.,
\[\Pr\left[|N(X)|\leq (1-z_{n,m})\E[|N(X)|]\right]
\leq \exp\left(-\frac{z_{n,m}^2\left(\E[|N(X)|]\right)^2}{2d_{n,m}|X|}\right)
= \exp\left(-\frac{z_{n,m}^2d_{n,m}|X|}{8}\right) \leq (n+m)^{-4|X|}.\]
Also, when $|X|<m/d_{n,m}$, we have
\[
\left(1-\left(1-\frac{1}{m}\right)^{d_{n,m}|X|}\right)
\geq \frac{d_{n,m}|X|}{m}\left(1-\frac{d_{n,m}|X|}{m}\right),
\]
thus with probability $1-(n+m)^{-4|X|}$,
\[|N(X)|\geq (1-z_{n,m})\E[|N(X)|]\geq
(1-z_{n,m})d_{n,m}|X|\left(1-\frac{d_{n,m}|X|}{m}\right);
\]
Similarly when $|X|\geq m/d_{n,m}$ , we have
\[
\left(1-\left(1-\frac{1}{m}\right)^{d_{n,m}|X|}\right)\geq 1-e^{-1},
\]
and
\[|N(X)|\geq (1-z_{n,m})\E[|N(X)|]
\geq (1-z_{n,m})\left(1-e^{-1}\right)m\geq \left(1-z_{n,m}-e^{-1}\right)m.
\]
The proof for $Y\subset R$ is almost the same except that it's
sufficient to use Chernoff bound rather than Azuma's inequality
since the independence among the subjects in $N(Y)$.

\begin{equation*}
\begin{aligned}
\E[|N(Y)|]=n\left(1-\left(1-\frac{|Y|}{m}\right)^{d_{n,m}}\right)
\geq n\left(1-\exp\left(-\frac{d_{n,m}|Y|}{m}\right)\right)\geq \frac{nd_{n,m}|Y|}{2m}.
\end{aligned}
\end{equation*}
Using Chernoff bound, we can lower bound $|N(Y)|$, i.e.,
\[\Pr\left[|N(Y)|\leq (1-z_{n,m})\E[|N(Y)|]\right]
\leq \exp\left(-\frac{z_{n,m}^2\E[|N(Y)|]}{2}\right)
= \exp\left(-\frac{z_{n,m}^2nd_{n,m}|Y|}{4m}\right) \leq (n+m)^{-4|Y|}.\]
Thus when $|Y| < m / d_{n,m}$, with probability $1-(n+m)^{-4|Y|}$,
\[|N(Y)|\geq (1-z_{n,m})\E[|N(Y)|]\geq
(1-z_{n,m})\frac{nd_{n,m}|Y|}{m}\left(1-\frac{d_{n,m}|Y|}{m}\right);
\]
when $|Y|\geq m / d_{n,m}$, with probability $1-(n+m)^{-4|Y|}$,
\[|N(Y)|\geq (1-z_{n,m})\E[|N(Y)|]
\geq (1-z_{n,m})\left(1-e^{-1}\right)n\geq \left(1-z_{n,m}-e^{-1}\right)n.
\]
\end{proof}

\begin{lemma}[Local Growth of $\lowset_k$]
For $n$ and $m$ large enough, $k<K_{n,m}$ and a fixed subject $i\in \lowset_k$,
it holds w.p.\ $1-2(n+m)^{-4}$ that
\begin{itemize}
    \item If $k<K_{n,m}-1$,
        \[|N(\{i\})\cap \lowset_{k+1}|\geq q_{n,m}|N(\{i\})|,\]
        where $q_{n,m}=\frac{c_{n,m}\log (\frac{n}{m}d_{n,m})}{5\log n}$
            and $c_{n,m}=\frac{\exp(-(\alpha_{n,m}+1))}{4}$ as previously defined;
    \item If $k=K_{n,m}-1$,
        \[|N(\{i\})\cap \lowset_{k+1}|\geq \frac{75}{81}|N(\{i\})|.\]
\end{itemize}
\label{thm:local-property}
\end{lemma}

\begin{proof}
Pick a subject $i\in \lowset_k$.
% From \eqref{thm:degree} we know that with probability $1-(n+m)^{-4}$,
% \[(1-z_{n,m})d_{n,m}\leq \left|N(\{i\})\right|\leq (1+z_{n,m})d_{n,m},\]
% where $z_{n,m}=\sqrt{\frac{32m\log (n+m)}{nd_{n,m}}}$.
For any task assignment graph $G$ and its adjacency matrix $A$,
the corresponding adjacency matrix $A'$ of the exam result graph is a random variable of $A$.
Specifically, for any $A_{ij}=1$, $A_{ij}'$s are independent Bernoulli random variables
with probability $f(u_i-u_j)$ to be 1.
In other words, $\E[A_{ij}']=A_{ij}f(u_i-u_j)$.
By Chernoff bound,
\begin{equation*}
\begin{aligned}
&\Pr\left[\left|\sum_j A_{ij}'-\sum_jA_{ij}f(u_i-u_j)\right|
    \geq \sqrt{2|N(\{i\})|\log (n+m)}\right]\\
\leq &2\exp\left(-\frac{4|N(\{i\})|\log (n+m)}{|N(\{i\})|}\right)\leq 2(n+m)^{-4}.   
\end{aligned}
\end{equation*}
Below we use the above inequality and some analysis of function $f$
to count the number of subjects in $N(\{i\})\cap \lowset_{k+1}$.
The fact we use about function $f$ is
\begin{equation*}
\begin{aligned}
    & f'(x)=\frac{\exp(-x)}{(1+\exp(-x))^2}\leq \frac{1}{4}\quad \text{and}\\
    &f'(x)\geq \frac{\exp(-(\alpha_{n,m}+1))}{(1+\exp(-(\alpha_{n,m}+1)))^2}
    \geq \frac{\exp(-(\alpha_{n,m}+1))}{4}=c_{n,m},~\forall|x|\leq \alpha_{n,m}+1.
\end{aligned}
\end{equation*}
Thus for another subject $j$ such that $\varepsilon_j \leq \varepsilon_i$,
by mean value theorem, we have
\[f\left(u_i^*-u_j^*\right)-f\left(u_i-u_j\right)
    =f'(\xi_{ij})\left(\varepsilon_i-\varepsilon_j\right)
    \leq \frac{1}{4}\left(\varepsilon_i-\loweps\right)
    \leq \frac{D_k}{4},\]
where $\xi_{ij}\in\left[u_i-u_j,u_i^*-u_j^*\right]$.
Similarly, for a subject $j$ with $\varepsilon_j > \varepsilon_i + D_{k+1}-D_k$, we have
\[f\left(u_i-u_j\right)-f\left(u_i^*-u_j^*\right)
    =f'(\xi'_{ij})\left(\varepsilon_j-\varepsilon_i\right)
    \geq c_{n,m}(D_{k+1}-D_k),\]
where $\xi'_{ij}\in\left[u_i^*-u_j^*,u_i-u_j\right]$. Since
\[u_i-u_j-D_{K_{n,m}}\leq u_i-u_j - \left(\varepsilon_j-\varepsilon_i\right)
    =u_i^*-u_j^* \leq \xi'_{ij}\leq u_i-u_j,\]
and $D_{K_{n,m}}\rightarrow 0$ as $n,m\rightarrow\infty$ under condition \eqref{condition:consistency},
$|\xi'_{ij}|$ is bounded by $\alpha_{n,m}+1$ when $n$ and $m$ is large enough,
thus $f'(\xi'_{ij})\geq c_{n,m}$. Therefore, on the one hand,
\begin{equation}
\begin{aligned}
& \sum_{u_j^*-u_j>u_i^*-u_i}A_{ij}\left(f(u_i-u_j)-f(u_i^*-u_j^*)\right)\\
= &\sum_{j}A_{ij}\left(f(u_i-u_j)-f(u_i^*-u_j^*)\right)
    -\sum_{u_j^*-u_j\leq u_i^*-u_i}A_{ij}\left(f(u_i-u_j)-f(u_i^*-u_j^*)\right)\\
\leq & \sqrt{2N(\{i\})\log (n+m)}
    +\frac{1}{4}D_k\sum_{u_j^*-u_j\leq u_i^*-u_i}A_{ij}.
\end{aligned}
\label{thm:local:part1}
\end{equation}
On the other hand,
\begin{equation}
\begin{aligned}
& \sum_{u_j^*-u_j>u_i^*-u_i}A_{ij}\left(f(u_i-u_j)-f(u_i^*-u_j^*)\right)\\
\geq & \sum_{u_j^*-u_j> u_i^*-u_i+D_{k+1}-D_k}A_{ij}\left(f(u_i-u_j)-f(u_i^*-u_j^*)\right)\\
\geq & c_{n,m}(D_{k+1}-D_{k})\sum_{u_j^*-u_j> u_i^*-u_i+D_{k+1}-D_k}A_{ij}.
\end{aligned}
\label{thm:local:part2}
\end{equation}
Combining \eqref{thm:local:part1} and \eqref{thm:local:part2}, we have 
\[|N(\{i\})\cap \lowset_{k+1}|\geq
    \sum_{u_j^*-u_j\leq u_i^*-u_i+D_{k+1}-D_k} A_{ij}
    \geq \frac{c_{n,m}(D_{k+1}-D_k)-\sqrt{\frac{2m\log (n+m)}{(1-z_{n,m})nd_{n,m}}}}
        {c_{n,m}(D_{k+1}-D_k)+\frac{1}{4}D_k}|N(\{i\})|.\]
For $k<K_{n,m}-1$,
\[\frac{c_{n,m}(D_{k+1}-D_k)-\sqrt{\frac{2m\log (n+m)}{(1-z_{n,m})nd_{n,m}}}}
    {c_{n,m}(D_{k+1}-D_k)+\frac{1}{4}D_k}|N(\{i\})|
    \geq q_{n,m}|N(\{i\})|.\]
For $k=K_{n,m}-1$,
\[\frac{c_{n,m}(D_{k+1}-D_k)-\sqrt{\frac{2m\log (n+m)}{(1-z_{n,m})nd_{n,m}}}}
    {c_{n,m}(D_{k+1}-D_k)+\frac{1}{4}D_k}|N(\{i\})|\geq \frac{75}{81}|N(\{i\})|.\]
\end{proof}

\begin{proof}[Proof of Theorem \ref{thm:consistency}]
Denote $X_k=\lowset_k\cap L$ and $Y_k=\lowset_k\cap R$.
We inductively prove the following fact,
for $n$ and $m$ large enough,
with probability $1-(n+m)^{-2}$,
\begin{itemize}
    \item for $1\leq k\leq K_{n,m}-2$, and $k$ is odd,
        \[|X_k|,|Y_k|\geq \left(\frac{n}{m}d_{n,m}\right)^{(k-1)/2};\]
    \item for $1\leq k\leq K_{n,m}-2$, and $k$ is even,
        \[|X_k|\geq \left(\frac{n}{m}\right)^{k/2}d_{n,m}^{(k-1)/2},\quad 
        |Y_k|\geq \left(\frac{n}{m}\right)^{k/2-1}d_{n,m}^{(k-1)/2};\]
    \item for $k = K_{n,m}-1$,
        \[|X_k|,|Y_k|\geq \frac{m}{d_{n,m}};\]
    \item for $k = K_{n,m}$,
        \[|X_k|> \frac{n}{2},\quad |Y_k| > \frac{m}{2}.\]
\end{itemize}
We will use the following fact,
\begin{equation}
\begin{aligned}
|Y_{k+1}|\geq |N(X_k)\cap \lowset_{k+1}|&=|N(X_k)|-|N(X_k)\cap \overline{\lowset_{k+1}}|\\
&\geq |N(X_k)|-\sum_{i\in X_k}|N(\{i\})\cap \overline{\lowset_{k+1}}|\\
&=|N(X_k)|-\sum_{i\in X_k}\left(|N(\{i\})|-|N(\{i\})\cap \lowset_{k+1}|\right),\\
\end{aligned}
\label{thm:consistency:fact}
\end{equation}
and similarly
\[|X_{k+1}|\geq |N(Y_k)\cap \lowset_{k+1}|\geq |N(Y_k)|-\sum_{i\in Y_k}\left(|N(\{i\})|-|N(\{i\})\cap \lowset_{k+1}|\right),\]
to show the growth of $X_k$ and $Y_k$ respectively.

From now on we only consider $n$ and $m$ large enough.
Since $\lowelement\in \lowset_0$, w.l.o.g.\ we assume $|X_0|=1$.
if $X_0$ contains other subjects, we take a subset with size 1.
Then by fact \eqref{thm:consistency:fact}, \eqref{thm:degree} and
Lemma \ref{thm:local-property}, we know with probability $1-4(n+m)^{-4}$ that
\[|Y_1|\geq |N(X_0)\cap \lowset_{k+1}|\geq q_{n,m}|N(X_0)|>0.\]
% we prove with probability $??$, for $|X_k|\leq n/d_n^{3/2}$, we have $|Y_{k+1}|\geq d^{1/2}|X_k|$. 
For $1<k\leq K_{n,m}-2$, and odd $k$, we prove inductively.
We assume $|X_k|=\left(\frac{n}{m}d_{n,m}\right)^{(k-1)/2}$.
If $X_k$ is larger, we pick any subset with size $\left(\frac{n}{m}d_{n,m}\right)^{(k-1)/2}$.
Fact \eqref{thm:consistency:fact} show that
\[
|Y_{k+1}|\geq |N(X_k)|-\sum_{i\in X_k}\left(|N(\{i\})|-|N(\{i\})\cap \lowset_{k+1}|\right).
\]
By Lemma \ref{thm:vertex-expansion} and
union bound over all subset of $L$ with size $\left(\frac{n}{m}d_{n,m}\right)^{(k-1)/2}$,
it holds with probability $1-(n+m)^{-3|X_k|}$ that,
\[|N(X_k)|>(1-z_{n,m})\left(1-\frac{d_{n,m}|X_k|}{m}\right)d_{n,m}|X_k|.\]
By Lemma\ref{thm:local-property} and union bound over all possible subject $i\in X_k$,
it holds with probability $1-2(n+m)^{-3}$ that,
\[\forall i\in X_k,~|N(\{i\})\cap \lowset_{k+1}|\geq q_{n,m}|N(\{i\})|.\]

Therefore, with probability $1-3(n+m)^{-3}$ we have
\begin{equation*}
\begin{aligned}
|Y_{k+1}|
&\geq |N(X_k)|-\sum_{i\in X_k}\left(|N(\{i\})|-|N(\{i\})\cap \lowset_{k+1}|\right)\\
&\geq |N(X_k)|-(1-q_{n,m})\sum_{i\in X_k}|N(\{i\})|\\
&\geq (1-z_{n,m})\left(1-\frac{d_{n,m}|X_k|}{m}\right)d_{n,m}|X_k|-
    (1-q_{n,m})d_{n,m}|X_k|\\
&\geq |X_k|\left(\frac{m}{n}d_{n,m}\right)^{1/2}
    \left(\left(\frac{n}{m}\right)^{1/2}(q_{n,m}-z_{n,m})d_{n,m}^{1/2}-
    (1-z_{n,m})\frac{\left(\frac{n}{m}\right)^{1/2}d_{n,m}^{3/2}|X_k|}{m}\right)\\
&= |X_k|\left(\frac{m}{n}d_{n,m}\right)^{1/2}
\left(\left(\frac{n}{m}\right)^{1/2}(q_{n,m}-z_{n,m})d_{n,m}^{1/2}-
(1-z_{n,m})\frac{\left(\frac{n}{m}d_{n,m}\right)^{3/2}|X_k|}{n}\right)\\
&\geq |X_k|\left(\frac{m}{n}d_{n,m}\right)^{1/2}
\left(\left(\frac{n}{m}\right)^{1/2}(q_{n,m}-z_{n,m})d_{n,m}^{1/2}-1\right)\\
\end{aligned}
\end{equation*}
where the last inequality holds because we assume
$|X_k|=\left(\frac{n}{m}d_{n,m}\right)^{(k-1)/2}$.
Finally, under condition \eqref{condition:consistency},
we have for large enough $n$ and $m$,
$\left(\frac{n}{m}\right)^{1/2}(q_{n,m}-z_{n,m})d_{n,m}^{1/2}-1
    \geq \left(\frac{n}{m}\right)^\frac{1}{2}$,
thus $|Y_{k+1}|\geq d_{n,m}^{1/2}|X_k|$.
The same calculation applies to the case of $1<k\leq K_{n,m}-2$ and even $k$.
Similarly, we can prove for $1<k\leq K_{n,m}-2$,
    $|X_{k+1}|\geq \frac{n}{m}\left(d_{n,m}\right)^{1/2}|Y_k|$.
Therefore, we finish the proof for all $k < K_{n,m}$.

Similarly for $k=K_{n,m}$ and large enough $n$ and $m$,
with probability $1-4(n+m)^{-3}$, 
\begin{equation*}
\begin{aligned}
|Y_{K_{n,m}}|
&\geq |N(X_{K_{n,m}-1})|-\sum_{i\in X_{K_{n,m}-1}}\left(|N(\{i\})|
    -|N(\{i\})\cap \lowset_{K_{n,m}}|\right)\\
&\geq |N(X_{K_{n,m}-1})|-\left(1-\frac{75}{81}\right)\sum_{i\in X_{K_{n,m}-1}}|N(\{i\})|\\
&\geq (1-z_{n,m}-e^{-1})m-\frac{6}{81}m\\
&> \frac{m}{2}.
\end{aligned}
\end{equation*}
The same proof applies for $|X_{K_{n,m}}|$.
To summarize, with probability $1-(n+m)^{-2}$,
$|X_{K_{n,m}}| > n/2$ and $|Y_{K_{n,m}}|> m/2$,
thus $|\lowset_{K_{n,m}}|> (n+m)/2$.
By symmetry, $|\highset_{K_{n,m}}|>(n+m)/2$ with probability $1-(n+m)^{-2}$.
Then with probability $1-2(n+m)^{-2}$,
at least one subject $i\in \lowset_{K_{n,m}}\cap \highset_{K_{n,m}}$
lies in both $\lowset_{K_{n,m}}$ and $\highset_{K_{n,m}}$.
By definition, subject $i$ satisfies
\[
\varepsilon_i-\loweps\leq D_{K_{n,m}}\quad \text{and}\quad 
\higheps-\varepsilon_i\leq D_{K_{n,m}},
\]
thus
\[\|\boldsymbol{u}^*-\boldsymbol{u}\|_\infty\leq \higheps-\loweps\leq 2D_{K_{n,m}},\]
which tends to 0 under condition \eqref{condition:consistency}.
\end{proof}

\begin{corollary}[Rates]
In the case where $\alpha_{n,m}=O(1)$, and $d_{n,m}=\Omega\left(\frac{m\log^3(n+m)}{n}\right)$, with probability $1-2(n+m)^{-2}$, we have
\[\|\boldsymbol{u}^*-\boldsymbol{u}\|_\infty=O\left(\frac{\log n}{\log (\frac{n}{m}d_{n,m})}\sqrt{\frac{m\log (n+m)}{nd_{n,m}}}\right).\]
\end{corollary}

\subsection{Analysis of Our Algorithm}

Our algorithm uses the MLEs to predict the student's performance within the component. Based on the consistency of the MLEs, we show the bias of our algorithm when Condition A is satisfied.
\begin{theorem}\label{thm:ex-post error}
When Condition A is satisfied, the exam result graph is strongly connected. In this case, the MLE is unique and we have
\[
\left(\mathrm{alg}_i-\mathrm{opt}_i\right)^2\leq \frac{1}{4}\|\boldsymbol{u}-\boldsymbol{u}^*\|_{\infty}^2.
\]
\end{theorem}

\begin{proof}
When the exam result graph is strongly connected, the algorithm calculates the MLEs $\boldsymbol{u}^*$ and gives student $i$ a grade of $\mathrm{alg}_i=\frac{1}{|Q|}\sum_{j\in Q}f(u_i^*-u_j^*)$, while the ground truth probability of answering a random question correctly is $\mathrm{opt}_i=\frac{1}{|Q|}\sum_{j\in Q}f(u_i-u_j)$. Thus we have
\begin{equation}
\begin{aligned}
|\mathrm{alg}_i-\mathrm{opt}_i|
= &\left|\frac{1}{|Q|}\sum_{j}f(u_i^*-u_j^*)-\frac{1}{|Q|}\sum_{j}f(u_i-u_j)\right|\\
\leq &\frac{1}{|Q|}\sum_{j}\left|f(u_i^*-u_j^*)-f(u_i-u_j)\right|\\
= &\frac{1}{|Q|}\sum_{j}\left|f'(\xi_{ij})\right|\left|\varepsilon_i-\varepsilon_j\right|\\
\leq &\frac{2}{n}\|\boldsymbol{u}-\boldsymbol{u}^*\|_{\infty}\sum_{j}\left|f'(\xi_{ij})\right|\\
\leq &\frac{1}{2}\|\boldsymbol{u}-\boldsymbol{u}^*\|_{\infty},
\end{aligned}
\end{equation}
where the third-to-last equality is because of the mean value theorem, the next-to-last inequality is because $\left|\varepsilon_i-\varepsilon_j\right|\leq 2 \|\boldsymbol{u}-\boldsymbol{u}^*\|_{\infty}$, and the last inequality is because $|f'(x)|\leq \frac{1}{4}$.
Thus $\left(\mathrm{alg}_i-\mathrm{opt}_i\right)^2\leq \frac{1}{4}\|\boldsymbol{u}-\boldsymbol{u}^*\|_{\infty}^2$.
\end{proof}

Next we discuss the performance of our algorithm on several extreme cases of the task assignment graph. For example, the extremely sparse cases when $N(\{i\})$ is mutually disjoint for each student $i$ or each student receives only $d=1$ question. Another example is that the task assignment graph is a complete bipartite graph. In all of the above cases, our algorithm gives the same grade as simple averaging.

\begin{theorem}
When the task assignment graph satisfies that $N(i)$ is mutually disjoint for each student $i$ or each student receives only $d=1$ question, our algorithm gives the same grade as simple averaging.
\end{theorem}

\begin{proof}
In both cases, the exam result graph satisfies that every SCC is a single point, thus the algorithm's output totally relies on cross-component predictions. For each student, the comparable components for each student are exactly the questions that student receives. Thus the algorithm gives the same prediction as the student's correctness on those questions. The prediction for remaining questions is the average accuracy on the assigned questions by the algorithm's rule for incomparable components. Therefore, the algorithm's grade for the student is exactly the same as simple averaging.
\end{proof}

\begin{theorem}
When the task assignment graph is a complete bipartite graph, our algorithm gives the same grade as simple averaging.
\end{theorem}

\begin{proof}
In this case, the output of the algorithm only relies on existing edges. It directly follows that the algorithm gives the same grade as simple averaging.
\end{proof}
\section{Experiments}

\subsection{Real-World Data}
We use the anonymous answer sheets from a previously administered exam
with $|S|=35$ students and $|Q|=22$ questions.
The task assignment graph of the exam is a complete bipartite graph, i.e., each student is assigned with all questions.
The corresponding exam result graph happens to be strongly connected,
thus we are able to infer student abilities and question difficulties (\Cref{fig:complete-data:ecdf}). Below we study results from counterfactual
subgraphs with real exam answers and from data generated according to the model
with the inferred abilities and difficulties.

\begin{figure}[htbp]
    \centering
    \includegraphics[width=0.5\textwidth]{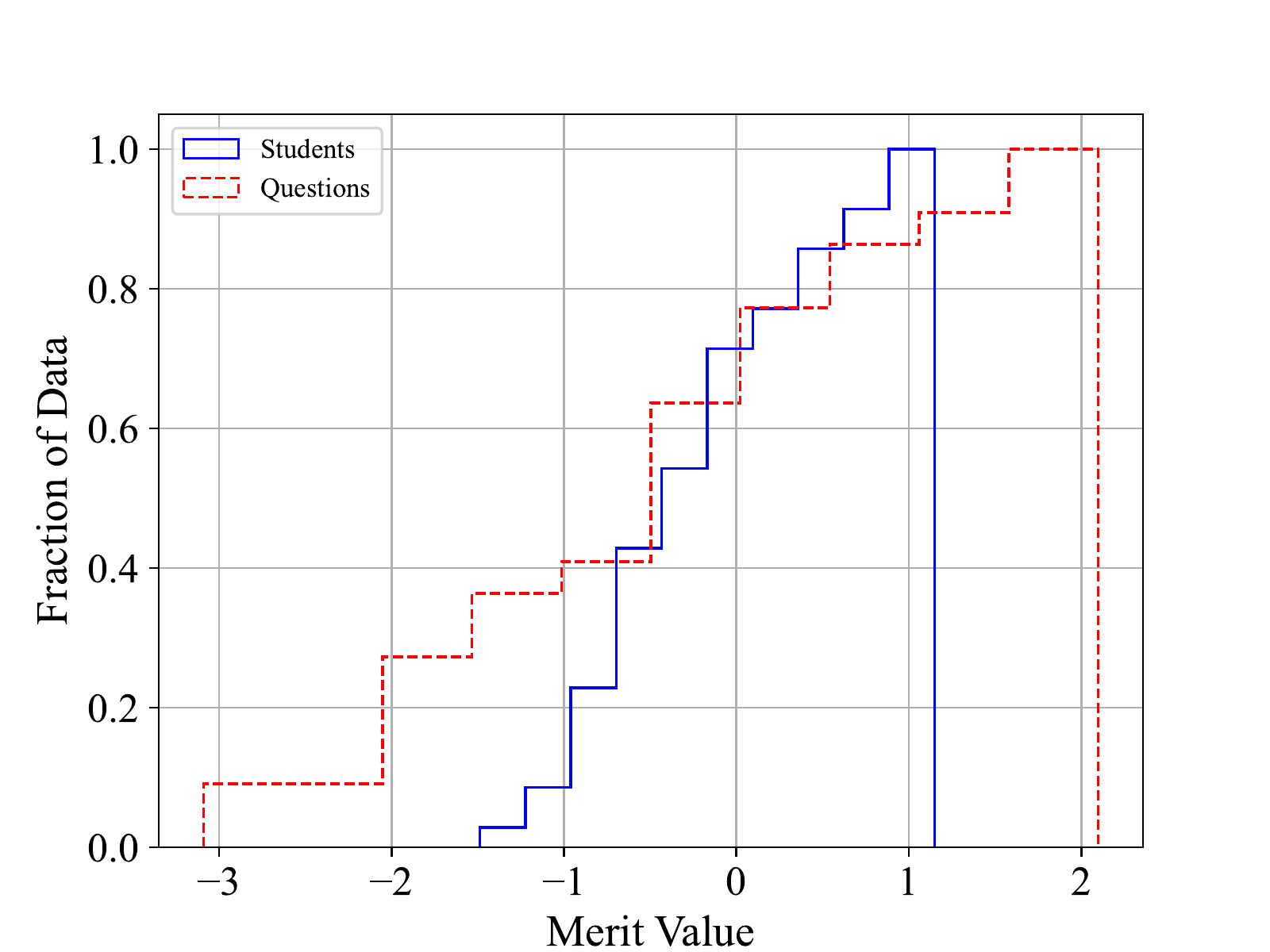}
    \caption{Empirical Cumulative Distribution of Merit Value.
        We analyze all students and questions under the Bradley-Terry-Luce model
        and show the empirical cumulative density function of inferred student abilities and question difficulties.
        The abilities ranges from -1.486 to 1.149 while the difficulties ranges from -3.090 to 2.099.}
    \label{fig:complete-data:ecdf}
\end{figure}

\subsection{Algorithms}

\paragraph{Simple Averaging} The grade for student $i$ is its average correctness on assigned questions.
See the formal definition in \Cref{example:Avg}.

\paragraph{Our Algorithm} The grade for student $i$ is an aggregation of the algorithm's prediction on her performance on each question.
All predictions can be classified into four cases, including existing edges (keep the fact as prediction),
same component (maximum likelihood estimators), comparable components (answer in line with the path direction)
and incomparable components (heuristic as simple averaging).  See the formal definition in \Cref{sec:method:grading rule}.

\subsection{Ex-post Bias}\label{sec:ex-post bias}

\subsubsection{Simulation 1: A Visualization of Simple Averaging's Ex-post Unfairness}\label{subsec:visual}
We compare the ex-post bias (\Cref{definition:ex-post bias}) between our algorithm and simple averaging
given a fixed random task assignment graph.
We use inferred parameters of all $35$ students and $22$ questions
according to \Cref{fig:complete-data:ecdf}.
The task assignment graph is generated by \Cref{alg:graph} with the parameter $m=22$ and $d=10$,
i.e.\ each student is assigned 10 random questions from the whole question bank.
The exam result graph is repeatedly generated according to the model.

\Cref{fig:fixed} shows the performance of two algorithms.
The left plot corresponds to our algorithm and the right plot corresponds to simple averaging.
In each plot, there are 35 confidence intervals,
each corresponding to the difference between the student's expected grade and her benchmark,
i.e. $\E_w[\mathrm{alg}_i]-\mathrm{opt}_i$.
The confidence intervals in the left plot are significantly closer to 0, compared to the right plot,
which visualizes the intuition that students are facing different overall question difficulties
under the random assignment and simple averaging fails to adjust their grades.
Instead, our algorithm infers the question difficulties and the student abilities and adjusts their grades accordingly, largely reducing the ex-post bias.

\begin{figure}[htbp]
    \centering
    \subfigure[Ex-post Grade Deviation of Our Algorithm]{
        \begin{minipage}{0.45\linewidth}
            \centering
            \includegraphics[width=\textwidth]{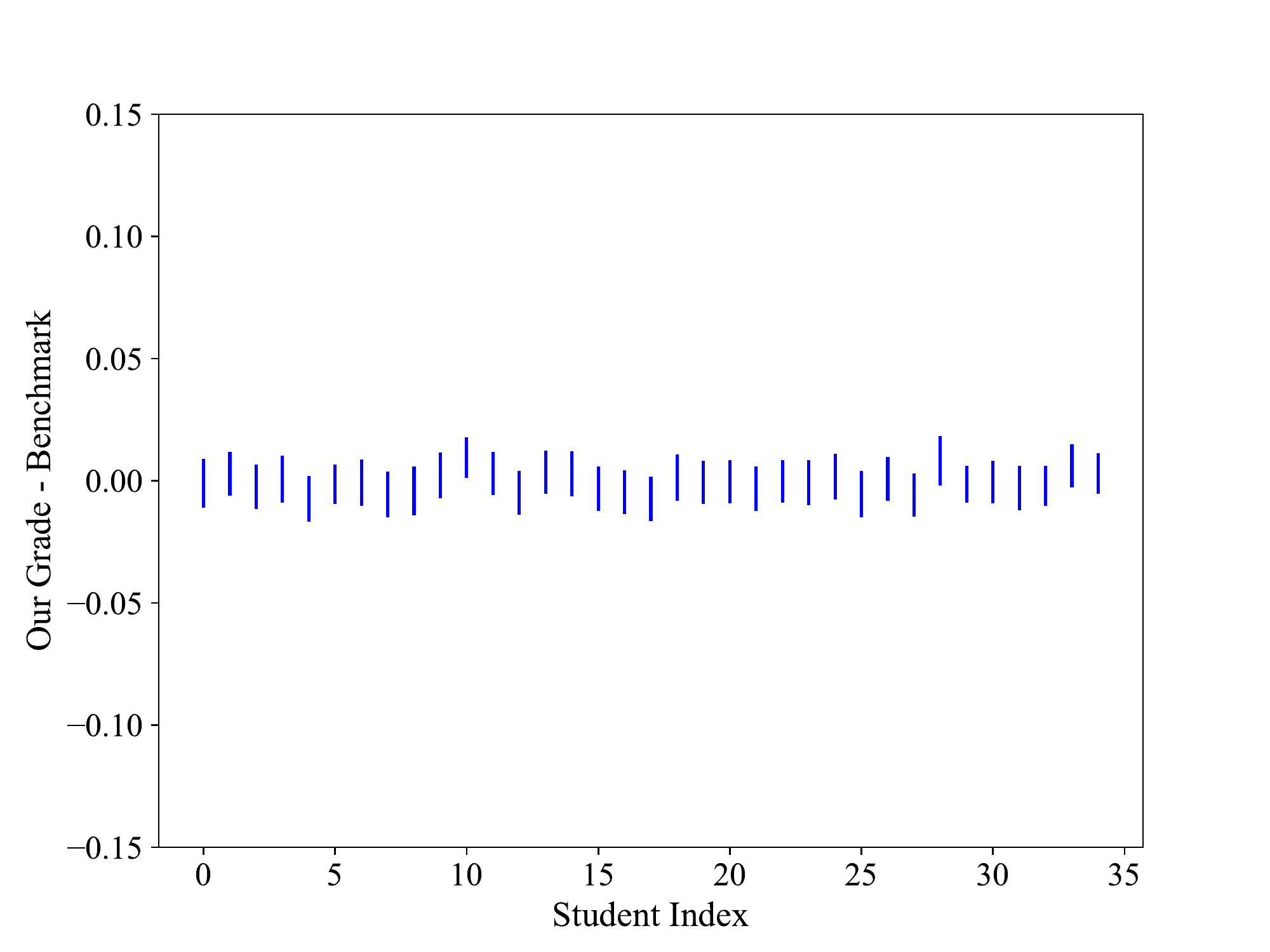}
        \end{minipage}
    }
    \subfigure[Ex-post Grade Deviation of Simple Averaging]{
        \begin{minipage}{0.45\linewidth}
            \centering
            \includegraphics[width=\textwidth]{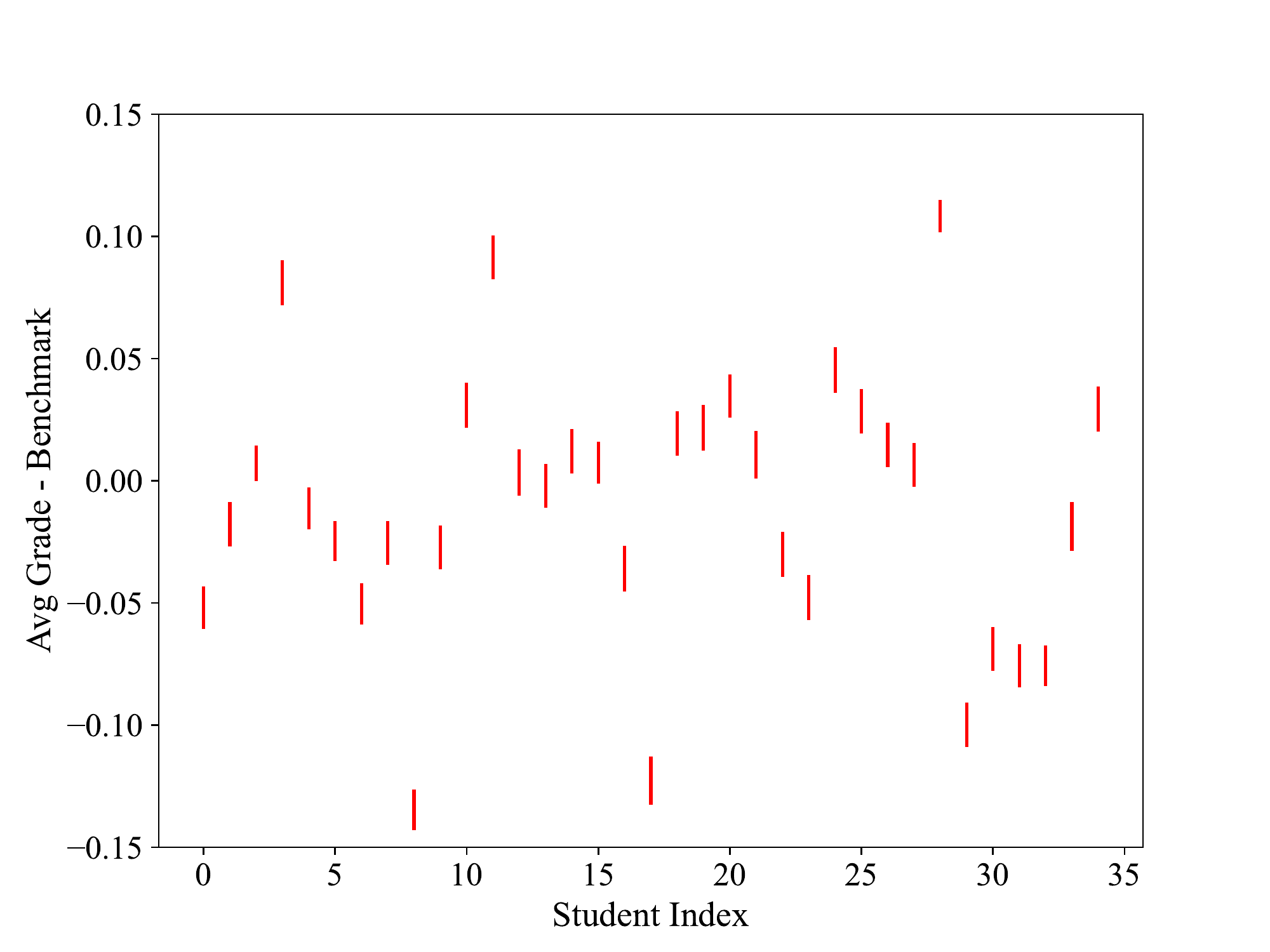}
        \end{minipage}
    }
    \caption{A Visualization of the Ex-post Grade Deviation with Degree Constraint $d=10$}
    \label{fig:fixed}
\end{figure}

\subsubsection{Simulation 2: The Effect of the Degree Constraint}

We compare the expected maximum ex-post bias, i.e., $\E_G\left[\max_{i\in S} \left(\E_w[\mathrm{alg}_i]-\mathrm{opt}_i\right)^2\right]$ and the expected average ex-post bias, i.e., $\E_G\E_{i\sim\mathcal U(S)}\left[\left(\E_w[\mathrm{alg}_i]-\mathrm{opt}_i\right)^2\right]$
between our algorithm and simple averaging.
We use inferred parameters of all $35$ students and $22$ questions
according to \Cref{fig:complete-data:ecdf}.
For each degree constraint $d$ from 1 to 22, we repeatedly generate task assignment graphs using algorithm~\Cref{alg:graph}
with the other parameter $m=22$, i.e. each student is assigned $d$ independent questions from the whole question bank.
For each task assignment graph, the exam result graph is repeatedly generated according to the model.

\Cref{fig:ex-post bias} shows two algorithms' expected maximum ex-post bias (\Cref{fig:ex-post bias:maximum})
and expected average ex-post bias (\Cref{fig:ex-post bias:average}) under different degree constraints,
where our algorithm (blue curve) outperforms simple averaging (red curve) on every degree constraint $d$.
Our algorithm's expected ex-post bias with the degree constraint $d=5$ is close to simple averaging's with the degree constraint $d=20$,
which means our algorithm can ask 15 fewer questions to each student to achieve the same grading accuracy as simple averaging.

\begin{figure}[htbp]
    \centering
    \subfigure[Expected Maximum Ex-post Bias]{
        \begin{minipage}{0.45\linewidth}
            \centering
            \includegraphics[width=\textwidth]{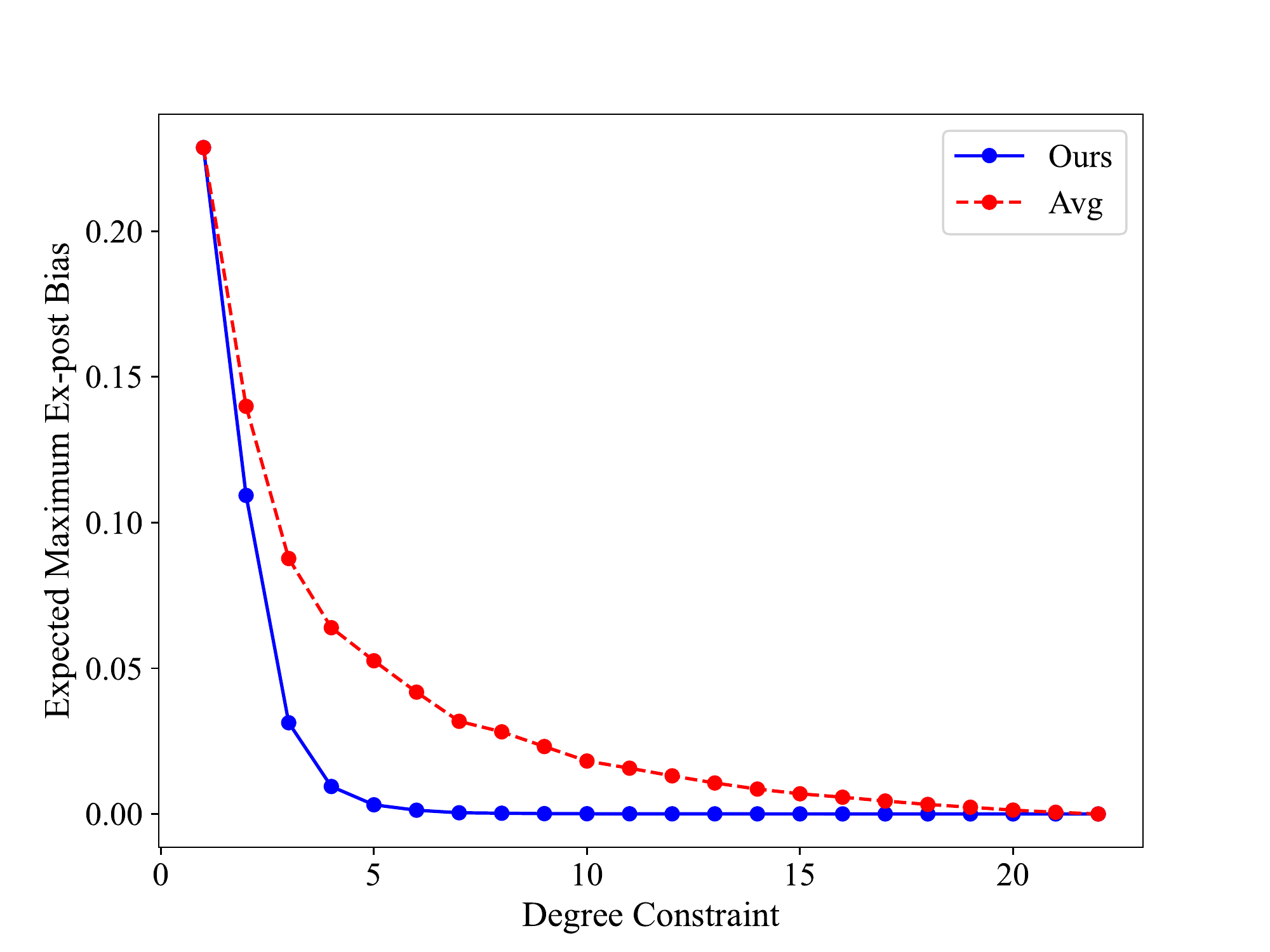}
        \end{minipage}
        \label{fig:ex-post bias:maximum}
    }
    \subfigure[Expected Average Ex-post Bias]{
        \begin{minipage}{0.45\linewidth}
            \centering
            \includegraphics[width=\textwidth]{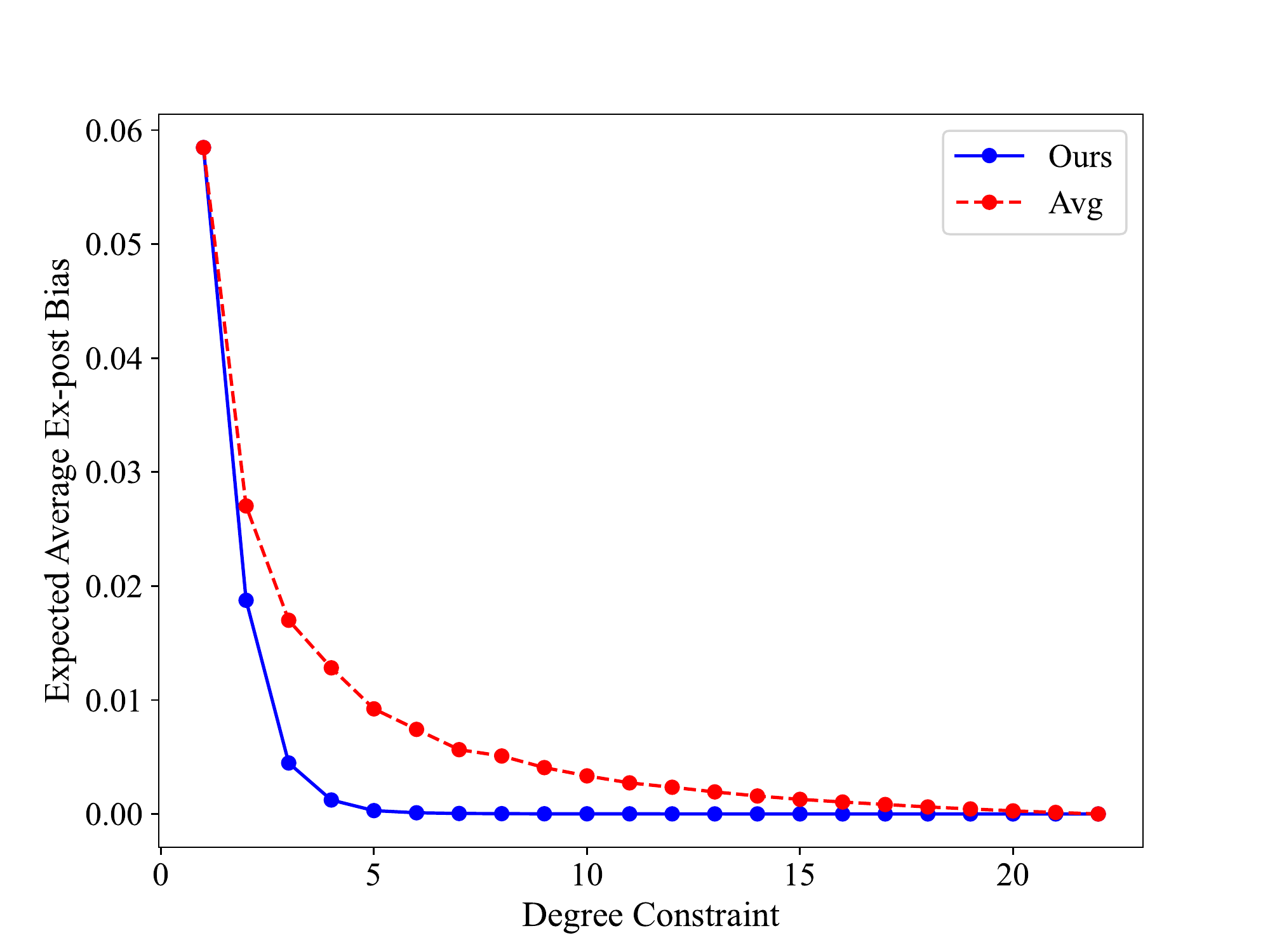}
        \end{minipage}
        \label{fig:ex-post bias:average}
    }
    \caption{Expected Aggregated Ex-post Bias v.s. Degree Constraint. The scale of expected average ex-post bias is about 4 times smaller than the scale of expected maximum ex-post bias.}
    \label{fig:ex-post bias}
\end{figure}

We emphasize that our algorithm does not only perform better in expectation over random graphs but also on each of them.
If we zoom in on a specific degree constraint, we can see two algorithms' maximum ex-post bias and average ex-post bias
on every task assignment graph.
\Cref{fig:random} shows the case with the degree constraint $d=10$.
\Cref{fig:random:maximum} corresponds to the maximum ex-post bias and \Cref{fig:random:average} corresponds to the average ex-post bias.
In each case, the left plot contains 100 points, corresponding to a different task assignment graph,
whose x-axis is the aggregated ex-post bias of simple averaging and whose y-axis is that of our algorithm;
the right plot is the histogram of the difference in the aggregated ex-post bias between our algorithm and simple averaging.
The simulation results show that our algorithm has a negligible aggregated ex-post bias
compared to simple averaging on every task assignment graph.

\begin{figure}[htbp]
    \centering
    \subfigure[Expected Maximum Ex-post Bias: 2.43e-4 for Our Algorithm and 1.80e-2 for Simple Averaging]{
        \begin{minipage}{0.8\linewidth}
            \centering
            \includegraphics[width=\textwidth]{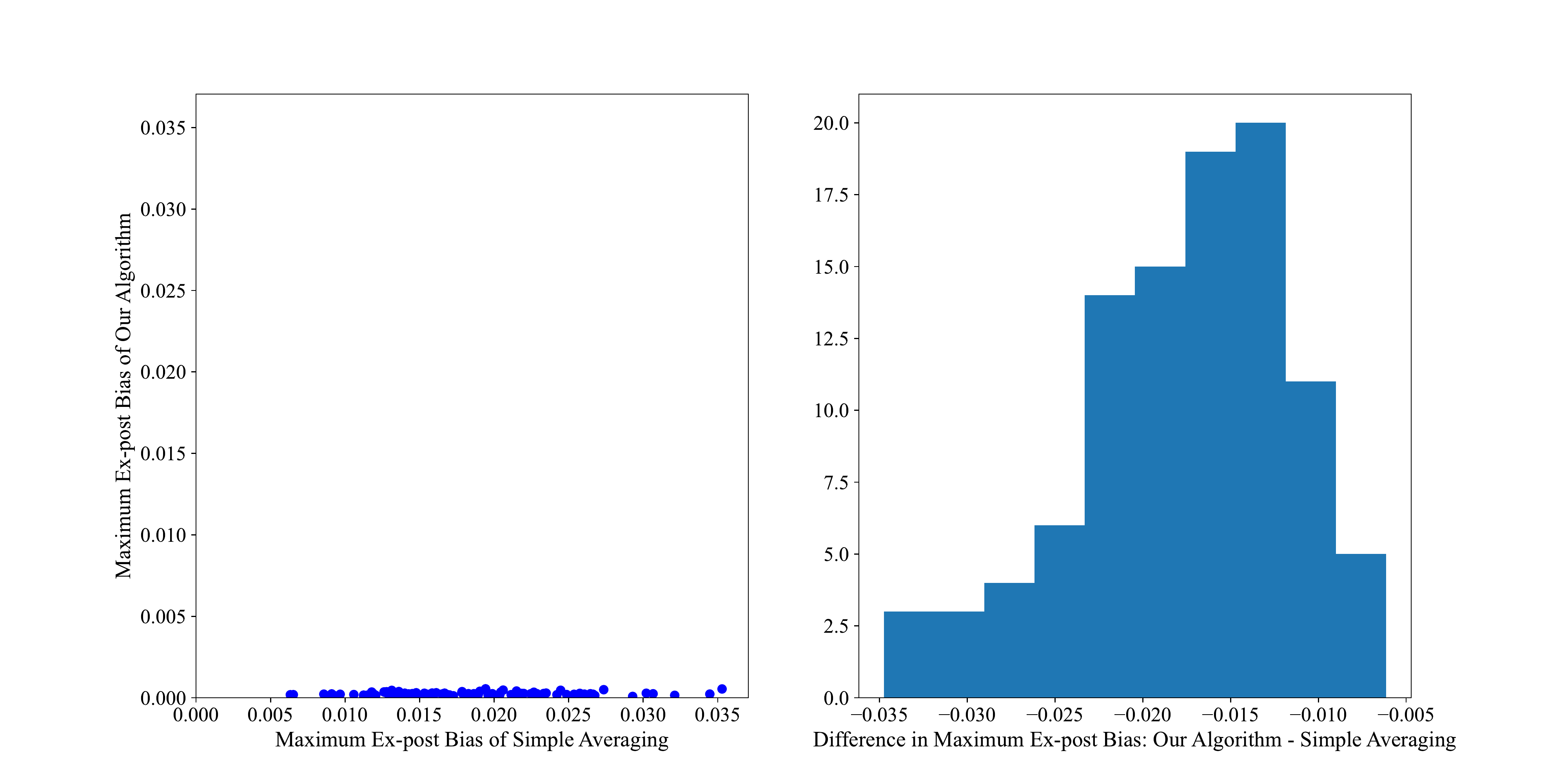}
        \end{minipage}
        \label{fig:random:maximum}
    }
    \subfigure[Expected Average Ex-post Bias: 4.08e-5 for Our Algorithm and 3.31e-3 for Simple Averaging]{
        \begin{minipage}{0.8\linewidth}
            \centering
            \includegraphics[width=\textwidth]{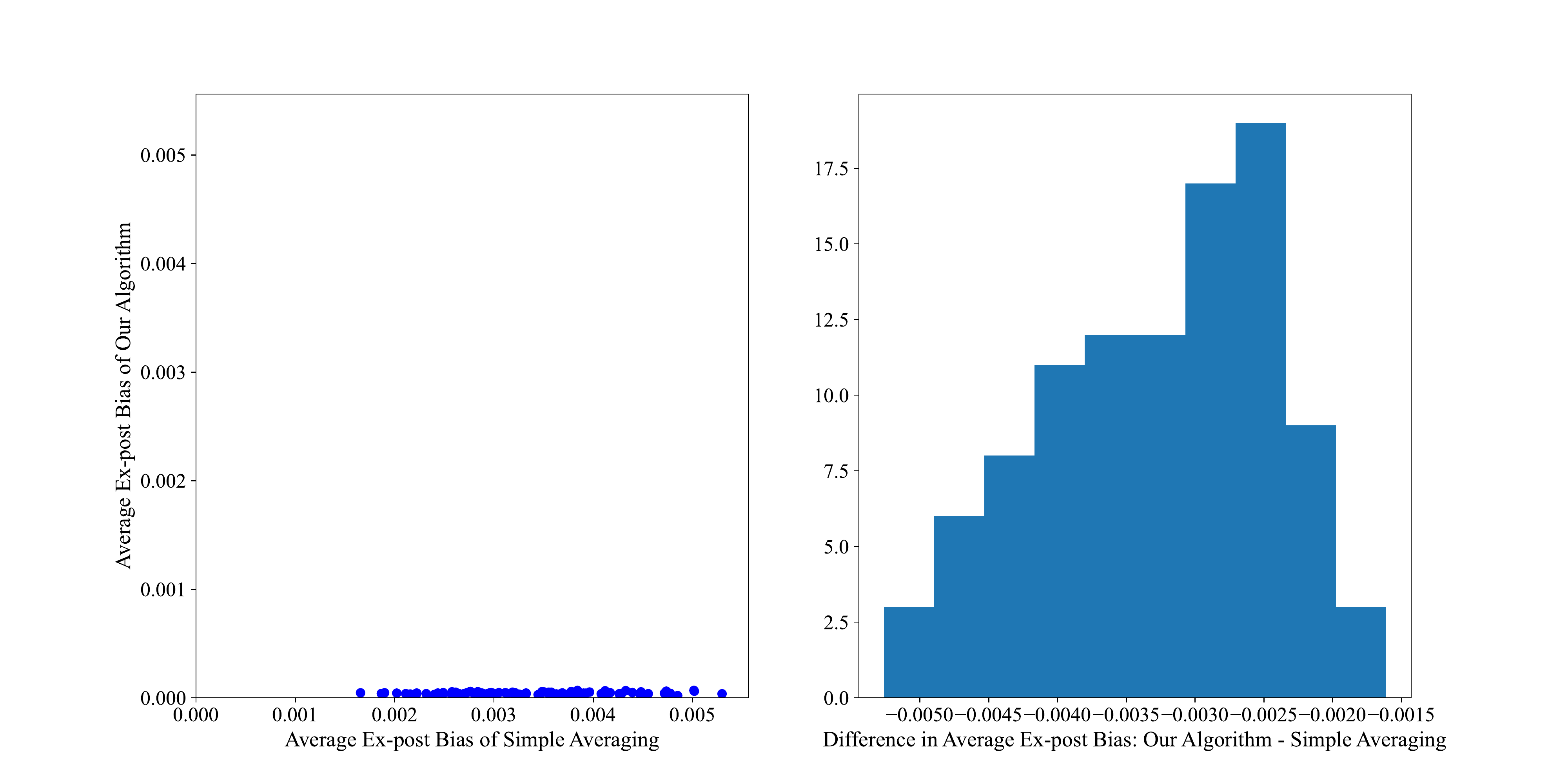}
        \end{minipage}
        \label{fig:random:average}
    }
    \caption{Aggregated Ex-post Bias on Task Assignment Graphs with Degree Constraint $d=10$}
    \label{fig:random}
\end{figure}

\subsection{Ex-post Error and Bias-Variance Decomposition}\label{sec:ex-post error}
In this part, we are investigating the expected average ex-post error (\Cref{definition:ex-post error}), i.e.,\\
$\E_G\E_i\E_w[\left(\mathrm{alg}_i-\mathrm{opt}_i\right)^2]$. Through
bias-variance decomposition,
we relate the ex-post error to the ex-post bias
and the variance in the algorithm performance.
\begin{theorem}[Bias-Variance Decomposition]\label{thm:bias-variance decomposition}
\[\E_G\E_i\E_w[\left(\mathrm{alg}_i-\mathrm{opt}_i\right)^2]=\E_G\E_i[\left(\E_w[\mathrm{alg}_i]-\mathrm{opt}_i\right)^2]+\E_G\E_i\E_w[\left(\mathrm{alg}_i-\E_w[\mathrm{alg}_i]\right)^2].\]
\end{theorem}

\begin{proof}
We prove a stronger argument of the decomposition for any fixed student $i$ and any fixed task
assignment graph $G$,
\begin{align*}
&\forall i,G,~\E_w[\left(\mathrm{alg}_i-\mathrm{opt}_i\right)^2]\\
=&\E_w[\left(\mathrm{alg}_i-\E_w[\mathrm{alg}_i]+\E_w[\mathrm{alg}_i]-\mathrm{opt}_i\right)^2]\\
=&\left(\E_w[\mathrm{alg}_i]-\mathrm{opt}_i\right)^2+\E_w[\left(\mathrm{alg}_i-\E_w[\mathrm{alg}_i]\right)^2]+2\E_w[\left(\mathrm{alg}_i-\E_w[\mathrm{alg}_i]\right)\left(\E_w[\mathrm{alg}_i]-\mathrm{opt}_i\right)]\\
=&\left(\E_w[\mathrm{alg}_i]-\mathrm{opt}_i\right)^2+\E_w[\left(\mathrm{alg}_i-\E_w[\mathrm{alg}_i]\right)^2]+2\left(\E_w[\mathrm{alg}_i]-\mathrm{opt}_i\right)\E_w[\left(\mathrm{alg}_i-\E_w[\mathrm{alg}_i]\right)]\\
=&\left(\E_w[\mathrm{alg}_i]-\mathrm{opt}_i\right)^2+\E_w[\left(\mathrm{alg}_i-\E_w[\mathrm{alg}_i]\right)^2].
\end{align*}
\end{proof}

With the same setting in \Cref{subsec:visual}, we show the expected average ex-post error of
our algorithm and simple averaging in \Cref{table: real-world parameter}. Our algorithm achieves
a factor of 8 percent smaller ex-post error in total. But after the decomposition, we can see that our
algorithm achieves a factor of 99 percent smaller ex-post bias with the cost of a factor of 10 percent larger 
variance. In practice, students will only take the exam once, so inevitably the variance of
the algorithm would contribute to the total error. Our algorithm does not focus on how to
reduce variance over the noisy answering process, instead, it focuses on the expected performance
of the algorithm, i.e., it makes the ex-post bias much closer to zero. To verify that our algorithm
does not increase the variance too much, we also run the simulation under ``the worst case'' of our
algorithm, i.e., all students have the same abilities and all questions have the same difficulties.
In this setting, our algorithm faces the risk of over-fitting, while simple averaging works perfectly. In \Cref{table: all-same parameter}, we can see that both algorithms achieve ex-post
biases close to 0, and our algorithm has a factor of 1.6 percent larger variance than simple averaging which is the main contribution to the difference in ex-post errors.

\begin{center}
\begin{table}[htbp]
\centering
\begin{tabular}{|p{0.15\textwidth}|p{0.15\textwidth}|p{0.15\textwidth}|p{0.15\textwidth}|}
\hline 
 & Ex-post Bias & Variance & Ex-post Error \\
\hline 
 Ours & ~0.00004 & 0.0188 & ~0.0188 \\
\hline 
 Avg & ~0.00331 & 0.0170 & ~0.0203 \\
\hline 
 Ours-Avg & -0.00327 & 0.0018 & -0.0015 \\
 \hline
\end{tabular}
\caption{Bias-Variance Decomposition in the setting of real-world parameters}
\label{table: real-world parameter}
\end{table}
\end{center}

\begin{center}
\begin{table}[htbp]
\centering        
\begin{tabular}{|p{0.15\textwidth}|p{0.15\textwidth}|p{0.15\textwidth}|p{0.15\textwidth}|}
\hline 
 & Ex-post Bias & Variance & Ex-post Error \\
\hline 
 Ours & 0.0000500 & 0.0254 & 0.0255 \\
\hline 
 Avg & 0.0000493 & 0.0250 & 0.0250 \\
\hline 
 Ours-Avg & 0.0000007 & 0.0004 & 0.0005 \\
 \hline
\end{tabular}
\caption{Bias-Variance Decomposition in the setting of all-the-same parameters}
\label{table: all-same parameter}
\end{table}
\end{center}

\subsection{Real-World Data Experiment: Cross Validation} We cannot repeat an exam in the real world and check the ex-post bias of the algorithms.
Thus, we sample part of the data we have as a new exam result graph and use them to predict the students' actual average on the data.
We randomly split the real-world data into training data and test data.
Specifically, for a fixed student sample size $d_1$ and a degree constraint $d_2$, in each repetition,
we randomly sample $d_1$ students and randomly choose $d_2$ questions and corresponding answers for each student independently as the training data,
use our algorithm (Ours) and simple averaging (Avg) to predict every student's average accuracy on the whole question bank,
and calculate the mean squared error.
Formally, the mean squared error $\mathrm{MSE}$ is defined as
$\mathrm{MSE}=\E_{X,\Tilde{S}}\left[\frac{1}{|\Tilde{S}|}\sum_{i\in \Tilde{S}}\left(\mathrm{alg}_{i}-\frac{1}{|Q|}\sum_{j\in Q}w_{ij}\right)^2\right],$
where $X$ is the training set given a fixed degree $d$,
$\Tilde{S}$ is the sampled student set,
$\mathrm{alg}_i$ is student $i$'s grade given by the algorithm
and $w_{ij}$ is the correctness of student $i$'s answer to question $j$. This measurement is closer to ex-post error except that the answering process is not repeatedly drawn.

In \Cref{fig:grading-rule:MSE}, we fix the student sample size $d_1 = 35$, i.e., $\Tilde{S} = S$
and change the degree constraint $d_2$ from 1 to $22$ and show the curve of the logarithm of $\mathrm{MSE}$. Our algorithm performs better than simple averaging when the degree constraint $d_2$ is larger than 5 and has a factor of 16\% to 20\% smaller MSE
compared to simple averaging when the degree constraint $d_2$ is larger than 10.
In \Cref{fig:grading-rule:Threshold}, we consider for every possible student sample size $d_1$,
what the smallest degree constraints $d_2$ is for our algorithm to perform better than simple averaging.
It provides a reference for choosing the grading rule in different situations.

\begin{figure}[htbp]
    \centering
    \subfigure[Logarithm of MSE v.s.\ Degree Constraint]{
        \begin{minipage}[b]{0.45\linewidth}
        \centering
        \includegraphics[width=0.98\textwidth]{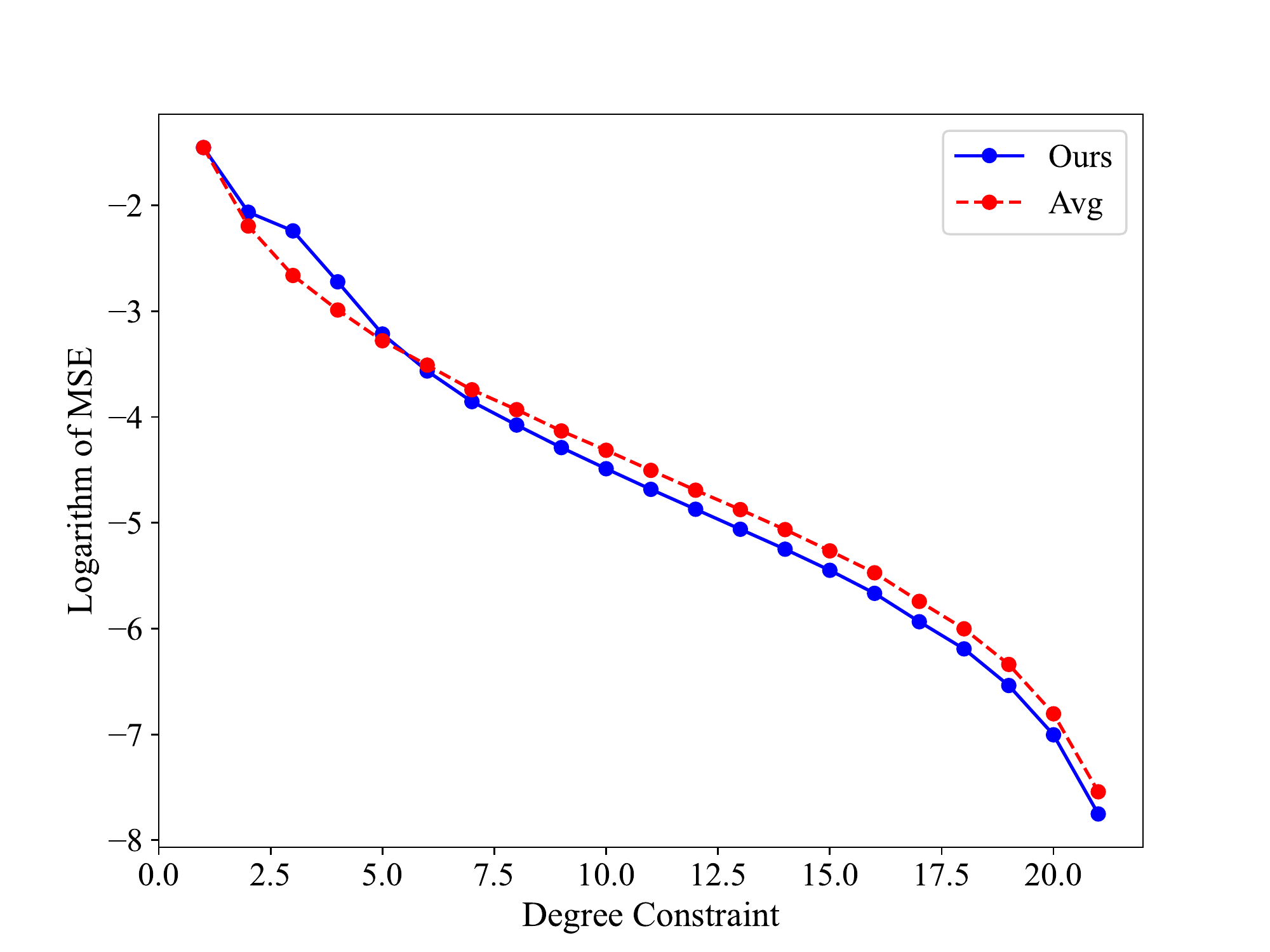}
        \end{minipage}
        \label{fig:grading-rule:MSE}
    }
    \hfill
    \subfigure[Threshold v.s.\ Student Sample Size]{
        \begin{minipage}[b]{0.45\linewidth}
        \centering
        \includegraphics[width=0.98\textwidth]{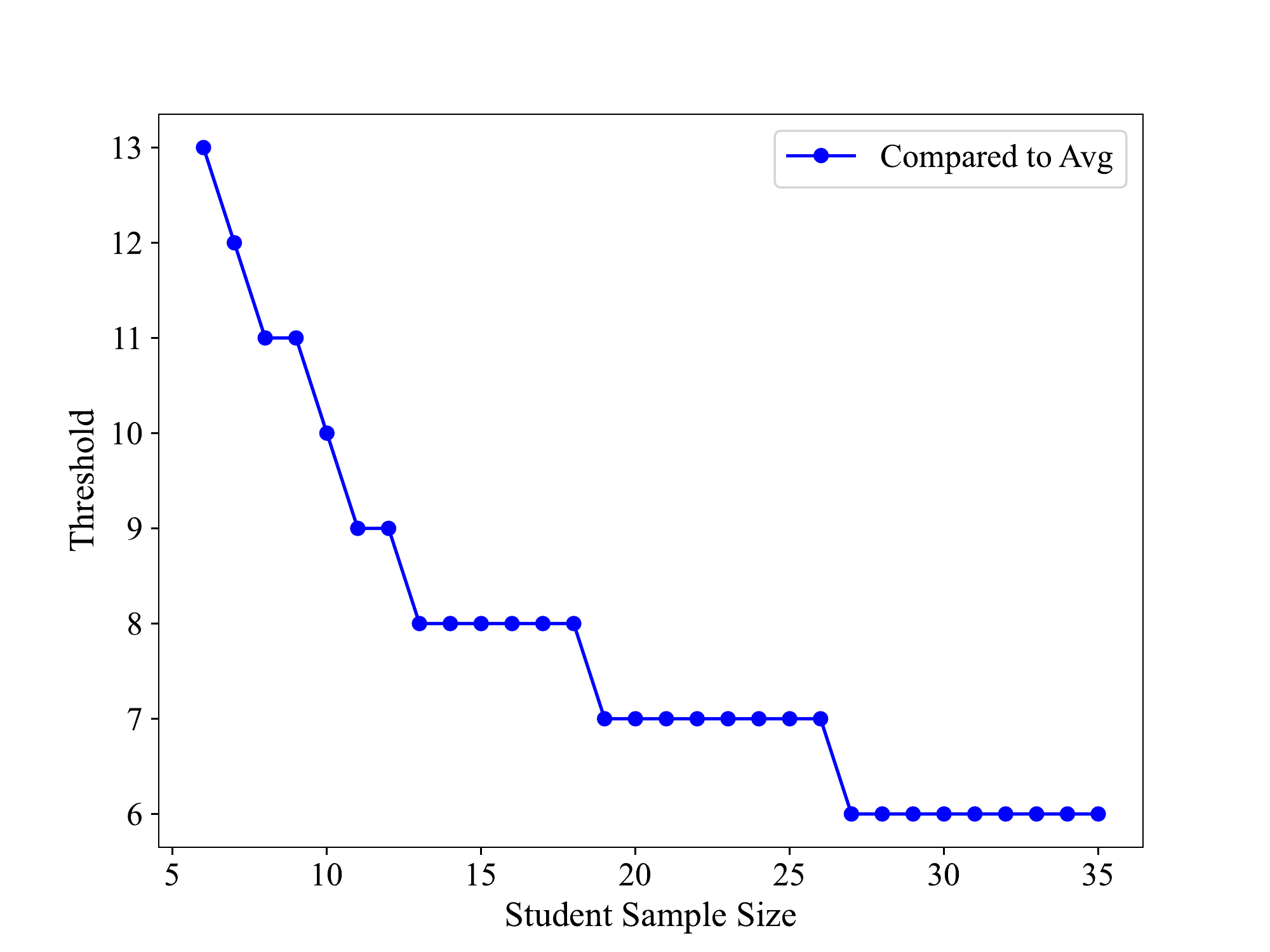}
        \end{minipage}
        \label{fig:grading-rule:Threshold}
    }
    \caption{Cross Validation}
    \label{fig:grading-rule}
\end{figure}

We also run the same cross-validation in numerical simulation. We assume two different Normal prior distributions for student abilities and question difficulties,
fitted by the data in \Cref{fig:complete-data:ecdf}.
For a fixed student sample size $d_1$ and degree constraint $d_2$, in every repetition,
we first draw $d_1$ i.i.d. student abilities and $|Q|=22$ i.i.d. question difficulties from those two prior distributions.
We use a complete task assignment graph to generate the exam result graph according to the model.
Then we randomly choose $d_2$ questions and the corresponding answers of each student as the training set
and use our algorithm and simple averaging to predict each student's average accuracy over the whole question bank.
Note that we do exact same things in the cross-validation,
so when calculating the MSE, we compare the algorithms' grades to the students' average accuracy given by the exam result graph
instead of the benchmark.
From \Cref{fig:simulated grading-rule} we observe that the simulation result is quite similar to that of the real-world cross-validation,
which suggests that the numerical simulation results we have in
\Cref{sec:ex-post bias} and \Cref{sec:ex-post error} could be a good reference for the
practical use of our algorithm.

\begin{figure}[htbp]
    \centering
    \subfigure[Logarithm of MSE v.s. Degree Constraint]{
        \begin{minipage}[b]{0.45\linewidth}
        \centering
        \includegraphics[width=0.98\textwidth]{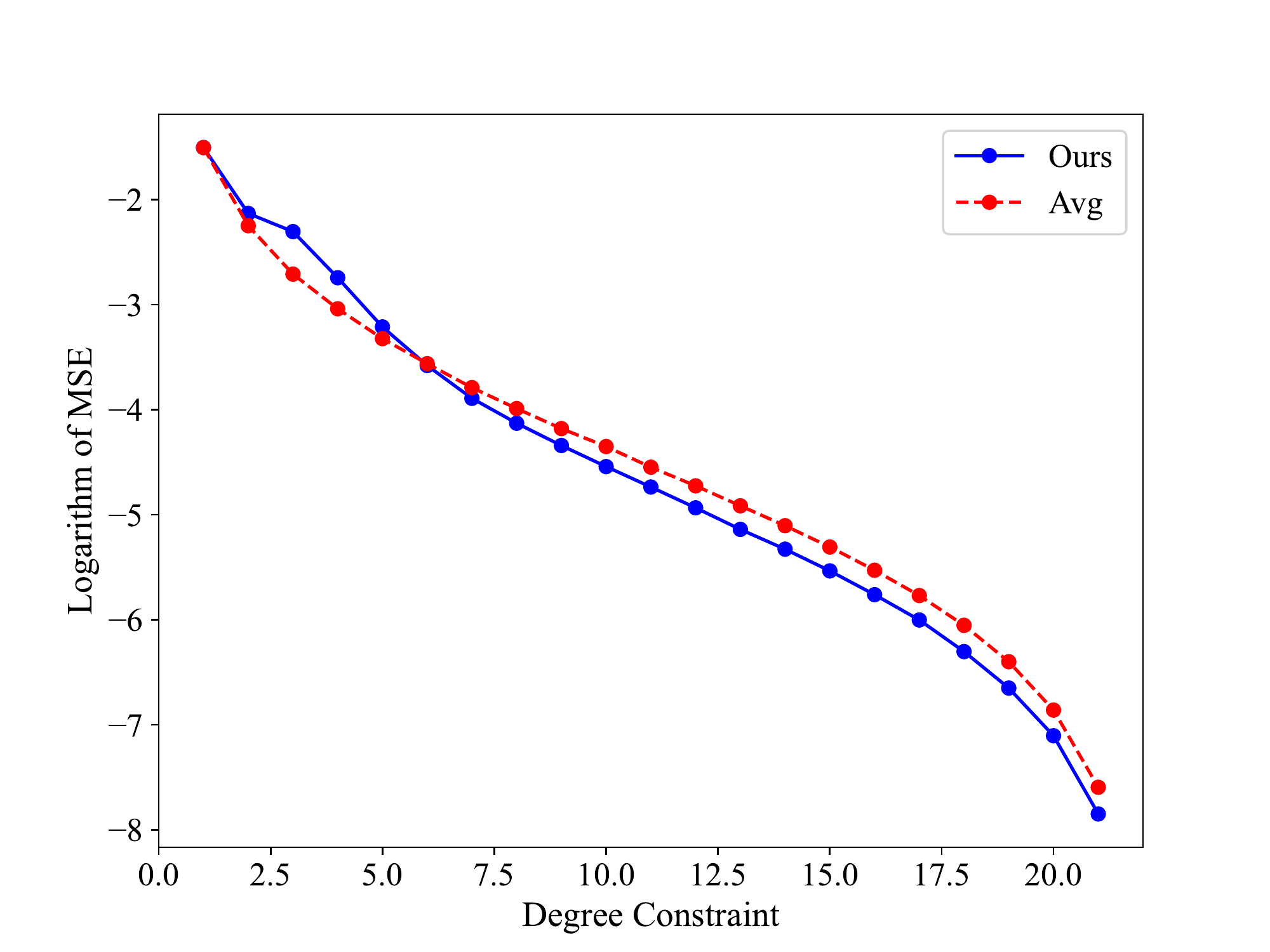}
        \end{minipage}
        \label{fig:simulated grading-rule:MSE}
    }
    \hfill
    \subfigure[Threshold v.s. Student Sample Size]{
        \begin{minipage}[b]{0.45\linewidth}
        \centering
        \includegraphics[width=0.98\textwidth]{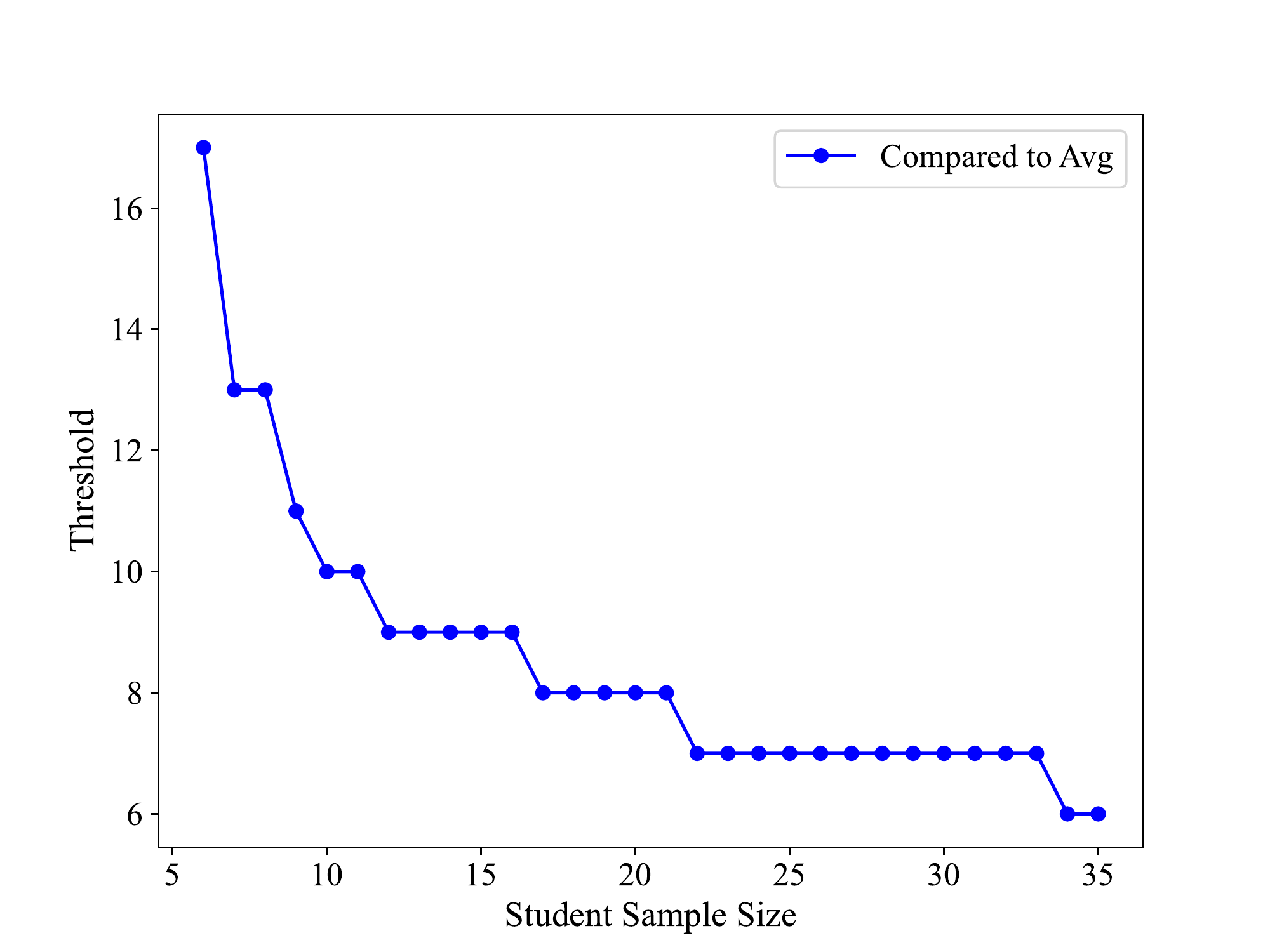}
        \end{minipage}
        \label{fig:simulated grading-rule:Threshold}
    }
    \caption{Simulated Cross Validation}
    \label{fig:simulated grading-rule}
\end{figure}

\subsection{The Effect of the Question Sample Size (Exam Design Problem)}
\label{sec:exam design problem}
We now consider the infinite question bank and
compare the expected maximum ex-post bias and the expected average ex-post bias between our algorithm and simple averaging.
We sampled and fixed $|S|=5$ random students among the previous $35$ students.
The question difficulties are i.i.d. distributed according to the linear interpolation of the difficulties
shown in \Cref{fig:complete-data:ecdf}.
For each question sample size $m$, we repeatedly generate task assignment graphs using \Cref{alg:graph} with the other parameter $d=5$.
It can be expected that when $m=d$ and $m\to\infty$, two algorithms should have the same performance.
\Cref{fig:infinite:small} shows a consistantly smaller expected maximum ex-post bias (\Cref{fig:infinite:small:maximum})
and expected average ex-post bias (\Cref{fig:infinite:small:average}) of our algorithm (blue curve) than simple averaging (red curve).
As the question sample size grows, the expected aggregated ex-post bias of our algorithm first decreases and then increases.
The turning point is about 6-9 for expected maximum ex-post bias and 10-15 for expected average ex-post bias.

\begin{figure}[htbp]
\centering
    \subfigure[Expected Maximum Ex-post Bias]{
        \begin{minipage}{0.45\linewidth}
            \centering
            \includegraphics[width=\textwidth]{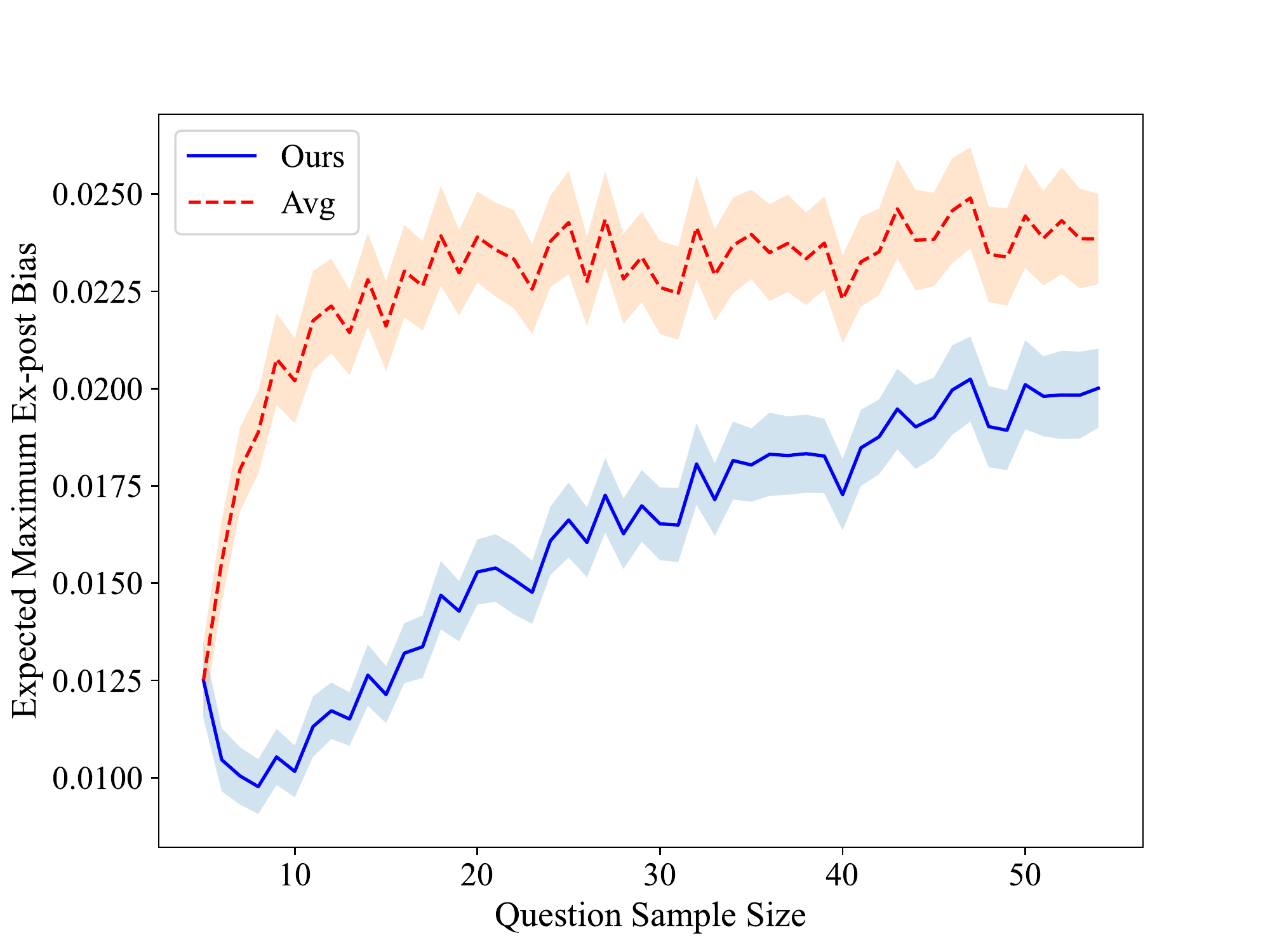}
        \end{minipage}
        \label{fig:infinite:small:maximum}
    }
    \subfigure[Expected Average Ex-post Bias]{
        \begin{minipage}{0.45\linewidth}
            \centering
            \includegraphics[width=\textwidth]{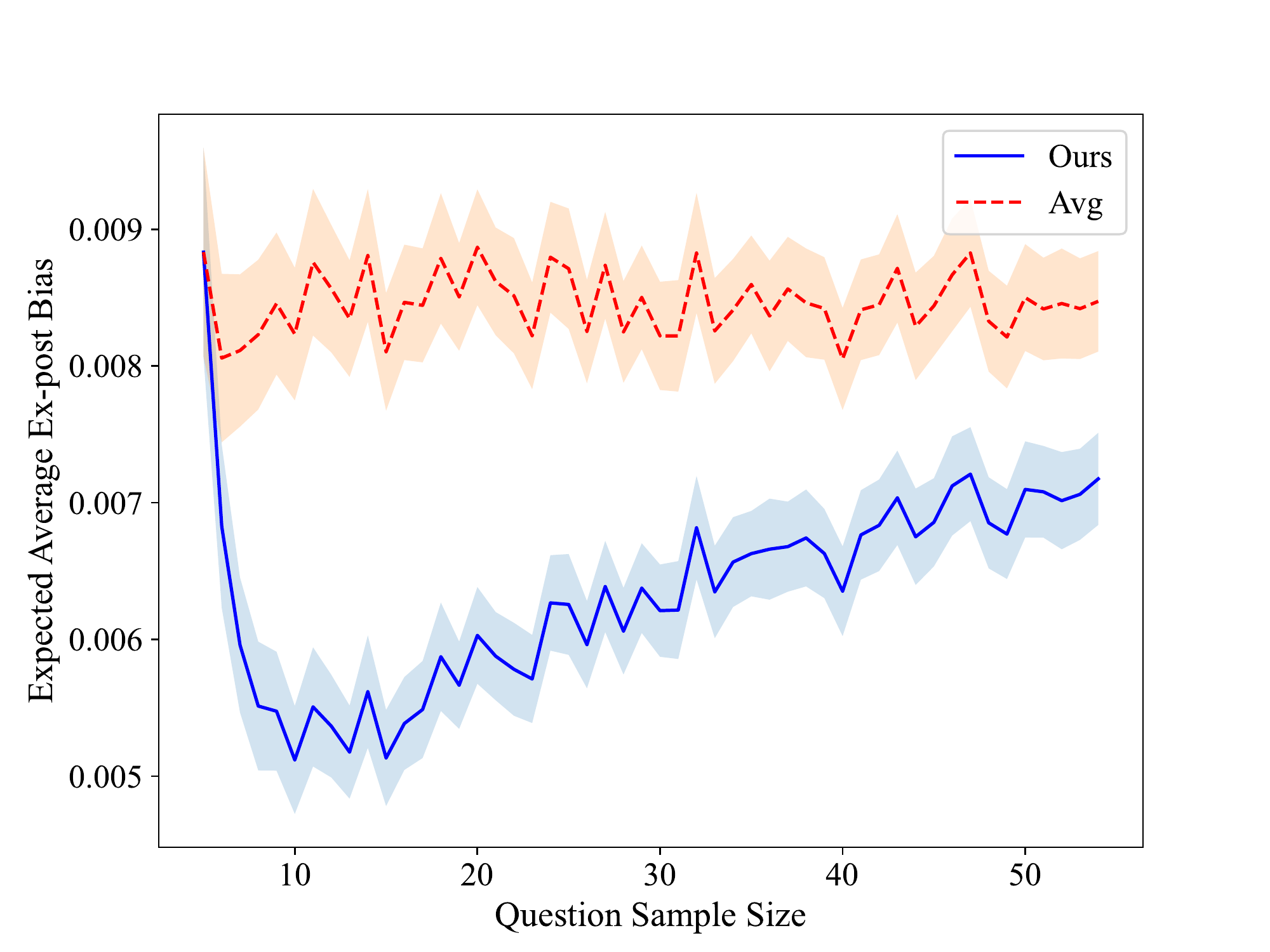}
        \end{minipage}
        \label{fig:infinite:small:average}
    }
    \caption{Expected Aggregated Ex-post Bias v.s. Question Sample Size}
    \label{fig:infinite:small}
\end{figure}
\section{Conclusions}\label{sec:conclusion}

We formulate and study the fair exam grading problem under the
Bradley-Terry-Luce model. We propose an algorithm that is a
generalization of the maximum likelihood estimation method. To
theoretically validate our algorithm, we prove the existence,
uniqueness, and uniform consistency of the maximum likelihood
estimators under the Bradley-Terry-Luce model on sparse bipartite
graphs. Our algorithm significantly outperforms simple averaging in
numerical simulation. On real-world data, our algorithm is better when
the students are assigned a sufficient number of questions (i.e., on
sufficiently long exams). We provide guidelines for how to choose
the grading rule given a certain number of students and a fixed exam
length.

Our model in this paper mainly considers true-or-false questions,
which can be extended to multiple-choice questions and to the case
where it can be assumed that students would guess if they cannot solve
a question. Our model treats student abilities and question difficulties
as one-dimensional, which can be extended to a multi-dimensional model that
takes different topics into account. Another potential extension of the model
is to introduce different groups of students, so each question might have
different difficulties for each group and we could ask for fairness across groups.
Our method to treat missing edges across comparable
components -- which predicts 0 or 1 -- needs to be improved, especially in the low-degree
environment (i.e., short exam lengths where the exam result graph is unlikely to be strongly connected). Also, it would be important to provide a simple and clear explanation to students for practical use.

\bibliography{ref.bib}
\bibliographystyle{plainnat}
\end{document}